\documentclass[journal]{IEEEtran}
\usepackage{cite} %CITATION PACKAGE
\usepackage{url} %URL PACKAGE
\usepackage{amssymb}
\usepackage{caption}
\usepackage{subcaption}

\usepackage{physics}
\usepackage{graphicx}% http://ctan.org/pkg/graphicx
\usepackage[export]{adjustbox}
\usepackage{hyperref}
\usepackage[english]{babel}
\usepackage{amsthm}
\usepackage{amsmath}
\usepackage{xcolor}

\allowdisplaybreaks

\newcommand{\Var}[1]{\operatorname{Var}\left[#1\right]}
\newcommand{\E}[1]{\operatorname{\mathbb{E}}\left[#1\right]}

\newcommand{\Cov}[1]{\operatorname{Cov}\left[#1\right]}

\newcommand{\Ib}{\boldsymbol{I}}

\newcommand{\R}{\mathbb{R}}
\newcommand{\Eb}{\mathbb{E}}

\newcommand{\Ll}{\mathcal{L}}
\newcommand{\Nl}{\mathcal{N}}
\newcommand{\Wl}{\mathcal{W}}

\newcommand{\normt}[1]{\| #1 \|_2}
\newcommand{\diag}{\operatorname{diag}}

\newcommand{\N}{\mathbb{N}}

\newcommand{\yb}{\boldsymbol{y}}

\newtheorem{theorem}{Theorem}[section]

\newtheorem{lemma}[theorem]{Lemma}

\newtheorem{remark}{Remark}
\begin{document}
\title{LDLT $\Ll$-Lipschitz Network Weight Parameterization Initialization}

\author{Marius~F.~R.~Juston$^{1}$,
        Ramavarapu~S.~Sreenivas$^{2}$,~\IEEEmembership{Senior Member,~IEEE,}
        Dustin~Nottage$^{3}$,
        Ahmet~Soylemezoglu$^{3}$
        % <-this % stops a space

\thanks{Marius F. R. Juston$^{1}$ is with The Grainger College of Engineering, Industrial and Enterprise Systems Engineering Department, University of Illinois Urbana-Champaign, Urbana, IL 61801-3080 USA (email: mjuston2@illinois.edu).}% <-this % stops a space

\thanks{Ramavarapu~S.~Sreenivas$^{2}$ is with The Grainger College of Engineering, Industrial and Enterprise Systems Engineering Department, University of Illinois Urbana-Champaign, Urbana, IL 61801-3080 USA (email: rsree@illinois.edu).}% <-this % stops a space

\thanks{Construction Engineering Research Laboratory$^{3}$, U.S. Army Corps of Engineers Engineering Research and Development Center, IL, 61822, USA}% <-this % stops a space

\thanks{This research was supported by the U.S. Army Corps of Engineers Engineering Research and Development Center, Construction Engineering Research Laboratory under Grant W9132T23C0013.}% <-this % stops a space
}

% The paper headers
\markboth{IEEE TRANSACTIONS ON XXX XXXX, XXX, XXX, September~2023}%
{Shell \MakeLowercase{\textit{et al.}}: Bare Demo of IEEEtran.cls for IEEE Journals}
% The second header will only appear for the odd-numbered pages after the title page when using the two-sided option.
% Note that you probably will NOT want to include the author's name in the headers of a peer-reviewed paper. You can use \ifCLASSOPTIONpeerreview for conditional compilation here if you desire.

% If you want to put a publisher's ID mark on the page, you can do it like this:
% \IEEEpubid{0000--0000/00\$00.00~\copyright~2015 IEEE}
% Remember, if you use this, you must call \IEEEpubidadjcol in the second column for its text to clear the IEEEpubid mark.

% Make the title area
\maketitle
\begin{abstract}
We analyze initialization dynamics for LDLT-based $\Ll$-Lipschitz layers by deriving the exact marginal output variance when the underlying parameter matrix $W_0\in \R^{m\times n}$ is initialized with IID Gaussian entries $\Nl(0,\sigma^2)$. The Wishart distribution, $S=W_0W_0^\top\sim\mathcal{W}_m(n,\sigma^2 \Ib_m)$, used for computing the output marginal variance is derived in closed form using expectations of zonal polynomials via James' theorem and a Laplace-integral expansion of $(\alpha \Ib_m+S)^{-1}$. We develop an Isserlis/Wick-based combinatorial expansion for $\E{\tr(S^k)}$ and provide explicit truncated moments up to $k=10$, which yield accurate series approximations for small-to-moderate $\sigma^2$. Monte Carlo experiments confirm the theoretical estimates. Furthermore, empirical analysis was performed to quantify that, using current He or Kaiming initialization with scaling $1/\sqrt{n}$, the output variance is $0.41$, whereas the new parameterization with $10/ \sqrt{n}$ for $\alpha=1$ results in an output variance of $0.9$.
The findings clarify why deep $\Ll$-Lipschitz networks suffer rapid information loss at initialization and offer practical prescriptions for choosing initialization hyperparameters to mitigate this effect. However, using the Higgs boson classification dataset, a hyperparameter sweep over optimizers, initialization scale, and depth was conducted to validate the results on real-world data, showing that although the derivation ensures variance preservation, empirical results indicate He initialization still performs better.
\end{abstract}
%
% Note that keywords are not normally used for peer-reviewed papers.
%
% For peer review papers, you can put extra information on the cover page as needed:
% \ifCLASSOPTIONpeerreview
% \begin{center} \bfseries EDICS Category: 3-BBND \end{center}
% \fi
% For peer review papers, this IEEEtran command inserts a page break and
% creates the second title. It will be ignored for other modes.
\IEEEpeerreviewmaketitle

\section{Introduction} \label{s_Intro}

\IEEEPARstart{T}{he} design of the $\Ll$-Lipschitz neural network provided a reliable solution to certifying the networks to be adversarially robust \cite{Nguyen2015, Szegedy2013, Biggio2013}, such that the decision output remains the same within a sphere of perturbation \cite{Tsuzuku2018}. For the design, multiple network architectures, formulation, functions have been proposed, ranging from utilizing Spectral Normalization (SN) \cite{Miyato2018, Roth2020}, Orthogonal Parametrization \cite{Trockman2021}, Convex Potential Layers (CPL) \cite{Meunier2022}, Almost-Orthogonal-Layers (AOL) \cite{Prach2022} and the recent SDP-based Lipschitz Layers (SLL) \cite{Araujo2023} and recently LDLT Layers \cite{Juston2025LDLTConstruction}. 
\par
This paper explores the impact of the weight parameterization of $\Ll$-Lipschitz networks employing LDLT Layers on the initialization of deep neural networks. This article aims to illuminate the properties and issues underlying weight normalization in LDLT networks, and to show how network parameterization affects feedforward variance mapping.
\begin{itemize}
    \item We derive the dynamics of the initialization LDLT initialization scheme utilized in \cite{Juston2025LDLTConstruction} to generate a Lipschitz network layer. 
    \item We derive the exact marginal total variance of the network input propagation for the LDLT network, assuming a normal distribution initialization for the underlying weight matrix
    \item Demonstrate and discuss how to tune the parameters $\alpha$ and $\sigma^2$ to achieve the best initialization for deep feedforward networks. 
    \item Given the derivation results, it is found that it is also not possible to derive a system of parameters that ensures that the system's marginal variance is equal to 1, demonstrating that deep feedforward networks would be affected by the current structure and formulation of the LDLT network's initialization scheme. Proof of this is derived and discussed.
\end{itemize}
The study of similar $\Ll$-Lipschitz network weight initialization schemes is derived from \cite{Juston20251-LipschitzProblem}, while the $\Ll$-Lipschitz network structure and LDLT parameterization are derived from \cite{Juston2025LDLTConstruction}.

\section{Related Work}

The initial works by Xavier \cite{Glorot2010} and Kaiming \cite{He2015} marked a pivotal moment for deep neural networks, establishing a methodology for properly initializing them to promote convergence, assuming hyperbolic tangent and ReLU \cite{Nair2010} activation functions. These works demonstrated that properly initializing deep feedforward networks can enable them to converge. Since then, modern machine learning has used Kaiming initialization for its networks, and modifications to the initialization gain have been activation-specific to ensure the stability criteria derived by Kaiming remain valid.
\par
In conjunction with the works for network initialization, \cite{Araujo2023} developed a unifying methodology to combine multiple existing $\Ll$-Lipschitz network structures into a unifying framework. This framework provides a guideline for creating a new, certifiably robust neural network. The authors achieve this by formulating feedforward networks as a nonlinear robust control Lur'e system \cite{lur1944theory} and enforcing SDP conditions on the weights of the generalized residual network structure. From this work, they can demonstrate general conditions for implementing a multilayered $\Ll$-Lipschitz network and combine previous works from Spectral Normalization (SN) \cite{Miyato2018, Roth2020}, Orthogonal Parameterization \cite{Trockman2021}, Convex Potential Layers (CPL) \cite{Meunier2022}, and Almost-Orthogonal-Layers (AOL) \cite{Prach2022} into a single constraint. From the framework, they generate an augmented version of the AOL with additional parameterization, called SDP-based Lipschitz Layers, which improves the network's generalizability.
\par
Most recently, the work of \cite{Juston2025LDLTConstruction} derived a generalization of the works of \cite{Araujo2023} by enabling the Lipschitz LMI structure to be expanded to different architectures through the LMI structure using an LDLT decomposition and constraining the block-diagonal matrix to maintain the necessary LMI positive definite condition. This end-to-end approach enables more complex architectures as long as they can be defined with exact $\Ll$-Lipschitz constraints. 
\par
Based on the weight parameterization schemed derived by \cite{Araujo2023, Prach2022} the work \cite{Juston20251-LipschitzProblem} demonstrated the issue with this type of network parameterization for feedforward networks where deep networks decay the input variance in a super-linear rate ensuring that deep network lose information between deep layers extremely fast, the paper quantities exactly the decay rate based on the network papers and explores the gradient backpropagation analysis as well.

\section{Normalization}

Based on paper \cite{Juston2025LDLTConstruction}, we utilize the following normalization methodology, which uses the specified formulation below,
\begin{lemma} \label{lm:weight_parameterization}
    If a matrix is parameterized as,
    \begin{align}
        M = \gamma W(\alpha \Ib + W^\top W)^{-\frac{1}{2}}, \label{eq:normalization}
    \end{align}
    for any $W \in \R^{\dim(M)}$, and $\gamma, \alpha > 0$ (can be parameterized using $\gamma = e^{\bar{\gamma}}, \bar{\gamma} \in \R$). Then $\normt{M} \leq \gamma$, \cite[Lemma 12]{Juston2025LDLTConstruction}.
\end{lemma}
As such, the research problem that we wish to explore is, assuming the parameterization scheme in \ref{lm:weight_parameterization}, what is the best initialization? In particular, we utilized an optimization in the computation where a Cholesky decomposition is used instead of computing the fully matrix inverse square root, defined in \cite[Lemma VII.2]{Juston2025LDLTConstruction} where,
\begin{lemma} \label{lm:cholesky_weight_parameterization}
    If a matrix is parameterized as,
    \begin{align*}
        M = \gamma WR^{-1},
    \end{align*}
    for any $W \in \R^{\dim(M)}$, and $\gamma, \alpha > 0$ (can be parameterized using $\gamma = e^{\bar{\gamma}}, \bar{\gamma} \in \R$), where $RR^\top = \alpha \Ib + W^\top W$. Then $\normt{M} \leq \gamma$.
\end{lemma}
We start by defining the forward pass and ensuring that the marginal distribution converges for a linear network. We want the variance to be stable across inputs and weight matrix dimensions. To perform this, we want to compute the marginal variance.

We assume an input vector $x \in \R^n$, whose elements are $x_i \sim \Nl(0, 1)$. We define the parameterized weight matrix as $W_0 \in \R^{m \times n}$ whose elements will be initialized as $W_0 \sim \Nl(0, \sigma^2)$. In turn, we wish to compute a linear network defined as,
\begin{align*}
    \yb = \bar W x = \gamma W_0(\alpha \Ib + W_0^\top W_0)^{-\frac{1}{2}} x = \gamma \underbrace{W_0 R^{-1}}_{\tilde W} x
\end{align*}
the variable, $\Var{y_i}$. To do this, we start by computing the total variance of $\yb$, the expectation square of the output vector,
\begin{align*}
    \E{\normt{\yb}^2} &= \E{\yb^\top \yb} = \E{\trace(\yb \yb^\top)}  = \trace \E{\yb \yb^\top} = \trace \Sigma_y.
\end{align*}
Given that each element is sampled from a central IID distribution, we have that,
\begin{align*}
    \E{\normt{\yb}^2} = \E{y_1^2 + \cdots + y_m^2} = \sum_{i=1}^m \E{y_i^2} = \sum_{i=1}^m \Var{y_i}
\end{align*}
such that,
\begin{align*}
    \Var{y}=\Var{y_1} = \cdots = \Var{y_m},
\end{align*}
which thus means that, 
\begin{align*}
    \trace \Sigma_y =  \sum_{i=1}^m \Var{y_i} =  m \Var{y}.
\end{align*}
We want to bound,
\begin{align*}
    \Var{y} = \frac{1}{m} \trace \Sigma_y,
\end{align*}
based on $\sigma$ such that $\Var{y} = 1$.

The marginal covariance $\Sigma_y$ was defined as,
\begin{align*}
    \Sigma_y = \Eb_{W_0}\left[\Cov{\yb | W_0} \right],
\end{align*}
the covariance is defined as, for fixed $\alpha$ and $\gamma$
\begin{align*}
    \Cov{\yb | W_0}&= \E{yy^\top | W_0},\\
    &= \E{\gamma^2 \tilde{W} x x^\top \tilde{W}^\top  | W_0}, \\
    &= \gamma^2 \tilde{W} \E{x x^\top} \tilde{W}^\top, \\
    &= \gamma^2 \tilde{W}\tilde{W}^\top, \\
    &= \gamma^2 W_0(\alpha \Ib_n + W_0^\top W_0)^{-1} W_0^\top.
\end{align*}
Using the Woodbury lemma, which states that,
\begin{theorem} \label{thm:woodbury_lemma}
    The Woodbury matrix identity, otherwise called the matrix inversion lemma, states that \cite{Woodbury1950InvertingMatrices}
    \begin{align*}
        (A + UCV)^{-1} = A^{-1} - A^{-1}  U(C^{-1} + V A^{-1}  U)^{-1}VA^{-1},  
    \end{align*}
    where $A \in \R^{n \times n}, U\in \R^{n \times k}, C\in \R^{k \times k}$ and $V\in \R^{k \times n}$ and $A$ is invertible.
\end{theorem} 

Where we simplify the system to,
\begin{lemma} \label{lm:reduced_woodbury}
    The following matrix inverses are equivalent
    \begin{align*}
        \alpha (\alpha \Ib + MM^\top)^{-1} = \Ib -  M(\alpha \Ib + M^\top M)^{-1}M^\top,
    \end{align*}
    given through the Woodbury matrix identity in Theorem \ref{thm:woodbury_lemma}. By setting $A = \alpha \Ib$, $C = \frac{1}{\alpha}\Ib$, $U= \sqrt{\alpha} M$ and $V = \sqrt{\alpha}M^\top$
\end{lemma}
As such we can transform the $\Cov{\yb | W_0}$ in the following way,
\begin{align*}
    \gamma^2 \Ib_m - \Cov{\yb | W_0} &= \gamma^2 \left(\Ib_m - W_0(\alpha \Ib_n + W_0^\top W_0)^{-1} W_0^\top\right), \nonumber \\
    &= \gamma^2 \alpha (\alpha \Ib_m + W_0 W_0^\top)^{-1}, \\
   \Cov{\yb | W_0} &=    \gamma^2 \Ib_m - \gamma^2 \alpha (\alpha \Ib_m + W_0 W_0^\top)^{-1}.
\end{align*}
Substituting back into the marginal covariance $\Sigma_y$, we have,
\begin{align*}
    \Sigma_y &= \Eb_{W_0}\left[\Cov{\yb | W_0} \right], \\
    &= \Eb_{W_0}\left[  \gamma^2 \Ib_m - \gamma^2 \alpha (\alpha \Ib_m + W_0 W_0^\top)^{-1} \right], \\
    &= \gamma^2 \Ib_m - \gamma^2 \alpha \Eb_{W_0}\left[  (\alpha \Ib_m + W_0 W_0^\top)^{-1} \right].
\end{align*}
Given that we wish to compute,
\begin{align}
    \Var{y} &= \frac{1}{m} \trace \Sigma_y, \\
    &= \frac{\gamma^2 }{m}(m - \alpha \trace \E{ (\alpha \Ib_m + W_0 W_0^\top)^{-1}}) ,\nonumber \\
    &= \frac{\gamma^2 }{m}(m - \alpha \E{ \trace ((\alpha \Ib_m + W_0 W_0^\top)^{-1})}) .\label{eqn:singular_variance}
\end{align}
Meaning that to derive the overall variance $\Var{y}$ we need to compute $\Eb_{W_0}\left[ (\alpha \Ib_m + W_0 W_0^\top)^{-1} \right]$. We set $S = W_0 W_0^\top$. Given that $W_0 \sim \Nl(0, \sigma^2)$ the resulting matrix $S$ is a $m\times m$ symmetric positive semi-definite matrix, following $S \sim \Wl_m(n, \sigma^2 \Ib_m)$, which is a central Wishart distribution with degrees of freedom $n$.
\begin{lemma}
For any Hermitian matrix $S \succeq 0$ and scalar $\alpha > 0$ 
\begin{align*}
    (\alpha \Ib + S)^{-1} &= \sum_{k = 0}^{\infty} (-1)^{k} \alpha^{-(k + 1)} S^{k}, \quad \text{whenever } \rho(S) < \alpha,
\end{align*}
and, regardless of $\rho(S)$, the exact Laplace representation,
\begin{align*}
     (\alpha \Ib + S)^{-1} &= \int_{0}^{\infty} e^{-\alpha t} e^{- S t} dt,
\end{align*}
always holds.
\end{lemma}
\begin{proof}
    Let $S$ be diagonalizable, which is true for every real symmetric matrix,
    \begin{align*}
        S = Q \diag(\lambda_1, \cdots, \lambda_n) Q^\top,
    \end{align*}
    then,  
    \begin{align*}
       (\alpha \Ib + S)^{-1} = Q  \diag\left(\frac{1}{\alpha + \lambda_1}, \cdots, \frac{1}{\alpha + \lambda_n}\right) Q^\top,
    \end{align*}
    the eigenvalues can then be represented through the Laplace integral,
    \begin{align*}
        \frac{1}{\alpha + \lambda} = \int_{0}^{\infty} e^{-\alpha t} e^{-\lambda t}dt,
    \end{align*}
    which can then be represented in matrix form as,
    \begin{align*}
        (\alpha \Ib + S)^{-1} = \int_{0}^{\infty} e^{-\alpha t} e^{-S t}dt,
    \end{align*}
    Now expanding the matrix exponential in its power series, we get that,
    \begin{align*}
        e^{-St} = \sum_{k = 0}^{\infty} \frac{(-t)^k}{k!} S^k,
    \end{align*}
    substituted in,
    \begin{align*}
        (\alpha \Ib + S)^{-1} &= \int_{0}^{\infty} e^{-\alpha t} e^{-S t}dt, \\
        &= \int_{0}^{\infty} e^{-\alpha t} \sum_{k = 0}^{\infty} \frac{(-t)^k}{k!} S^k dt, \\
        &= \sum_{k = 0}^{\infty} \left( \int_{0}^{\infty} e^{-\alpha t}  \frac{(-t)^k}{k!} dt \right) S^k ,\\
        &= \sum_{k = 0}^{\infty} (-1)^k \alpha^{-(k+ 1)} S^k,
    \end{align*}
    which only converges iff $\normt{S} < \alpha$.     Given that $\normt{S} \leq  \sigma^2 (\sqrt{m} + \sqrt{n} + t)^2$, for every $t \ge 0$ with probability at least $1- 2 e^{-t^2/2}$ \cite[Proposition 2.4]{Rudelson2010Non-asymptoticValues}. The moment assumption is thus only valid for a $\sigma^2 < \alpha / (\sqrt{m} + \sqrt{n} + t)^2 $.
\end{proof}
We perform the calculations using the unconditional integral form,
\begin{align*}
    \E{ \trace (\alpha \Ib_m + S)^{-1}} &= \int_{0}^{\infty} e^{-\alpha t} \sum_{k = 0}^{\infty} \frac{(-t)^k}{k!} \E{\trace(S^k)} dt,
\end{align*}
This representation holds for any positive-semidefinite $S$ and allows us to expand the matrix exponential $e^{-St}$ in terms of moments, exchanging the integral and the sum under expectation. The problem thus reduces to computing the matrix moments $\E{\trace(S^k)}$ for the Wishart matrix $S$.

\section{Expectation}

\subsection{James' expectation formula for zonal polynomials}
The key result we use \cite{James1964DistributionsSamples} is the expectation formula for zonal polynomials of a central Wishart:
\begin{theorem}
\label{thm:james}
If $X\sim\Wl_m(n,\Sigma)$ (central Wishart), then for any partition $\kappa$ with $|\kappa|=k$, \cite[Eqn. 24]{James1964DistributionsSamples},
\begin{equation}\label{eq:james}
\E{C_\kappa(X)} = 2^{k}\,\Big(\tfrac{n}{2}\Big)_\kappa C_\kappa(\Sigma).
\end{equation}
Define $\kappa$ as a partition $\kappa=(\kappa_1\ge\kappa_2\ge\cdots\ge0)$ with $|\kappa|:=\sum_i\kappa_i$. Let $C_\kappa(\cdot)$ be the zonal polynomial indexed by $\kappa$, \cite{James1964DistributionsSamples}, and $a$ be the scalar from which the multivariate (generalized) Pochhammer symbol is defined as,
\begin{align*}
(a)_\kappa &= \prod_{i=1}^{\ell(\kappa)}\Big(a-\tfrac{1}{2}(i-1)\Big)_{\kappa_i}, \\
(a)_k &:= a(a+1)\cdots(a+k-1).
\end{align*}
\end{theorem}
\noindent In particular, when $\Sigma=\sigma^2 I_m$ we use the homogeneity of the zonal polynomials:
\begin{align*}
C_\kappa(\Sigma) = \sigma^{2k}\, C_\kappa(I_m),
\end{align*}
so \eqref{eq:james} becomes
\begin{equation}\label{eq:james_sigma_identity}
\E{C_\kappa(X)} = 2^{k}\,\Big(\tfrac{n}{2}\Big)_\kappa \sigma^{2k}\, C_\kappa(I_m).
\end{equation}
\subsection{Expansion of power sums in the zonal basis}
Any symmetric polynomial in the eigenvalues of a symmetric (or real symmetric positive definite) matrix can uniquely be expanded as a linear combination in the zonal polynomial basis \cite[Theorem 7.2.5]{Muirhead2008AspectsTheory}. In particular, for fixed $k$ there exist constants $a_{k,\kappa}$ (depending only on the partition $\kappa$ and $m$) such that
\begin{equation}\label{eq:trace_power_zonal_expansion}
\tr(S^k) = \sum_{|\kappa|=k} a_{k,\kappa} C_\kappa(S).
\end{equation}
The coefficients $a_{k,\kappa}$ are determined by the change of basis from power-sum symmetric polynomials to zonal polynomials \cite{Macdonald1995SymmetricPolynomials}; closed forms exist for small $k$ and can be computed algebraically for larger $k$, \cite{James1964DistributionsSamples}.

\subsection{Expectation of $\tr(S^k)$ via James}
Taking expectation in \eqref{eq:trace_power_zonal_expansion} and using Theorem \ref{thm:james} yields the exact representation
\begin{align}\label{eq:Ek_traceSk_general}
\E{\tr(S^k)} &= \sum_{|\kappa|=k} a_{k,\kappa} \E{C_\kappa(S)}, \nonumber \\
 &= \sum_{|\kappa|=k} a_{k,\kappa} 2^{k}(\tfrac{n}{2})_\kappa C_\kappa(\Sigma).
\end{align}
For $\Sigma=\sigma^2 I_m$, we obtain the simplified form
\begin{equation}\label{eq:Ek_traceSk_sigmaI}
\E{\tr(S^k)} = 2^{k}\sigma^{2k} \sum_{|\kappa|=k} a_{k,\kappa} \Big(\tfrac{n}{2}\Big)_\kappa C_\kappa(\Ib_m).
\end{equation}
This expectation representation in \eqref{eq:Ek_traceSk_sigmaI} is exact. The difficulty in applying it for general $k$ is the computation of the basis coefficients $a_{k,\kappa}$ and the constants $C_\kappa(\Ib_m)$. For small $k$, these can be computed explicitly; for large $k$, combinatorial formulas for zonal polynomials and tables in the literature can be used.

Which gives the exact final $\Var{y}$ as,
\begin{align}
\Var{y} &= \frac{\gamma^2 }{m}\left(m - \alpha \int_{0}^{\infty} e^{-\alpha t} \sum_{k = 0}^{\infty} \frac{(-t)^k}{k!} \E{\trace(S^k)} dt \right), \nonumber \\
        &= \gamma^2 - \frac{\gamma^2 \alpha}{m} \int_{0}^{\infty} e^{-\alpha t} \sum_{k = 0}^{\infty} \frac{(-t)^k}{k!} \E{\trace(S^k)} dt, \nonumber \\
        &= \gamma^2 - \frac{\gamma^2 \alpha}{m} \int_{0}^{\infty} e^{-\alpha t} \nonumber \\ &\qquad \sum_{k = 0}^{\infty} \frac{(-2\sigma^2 t)^k}{k!} \sum_{|\kappa|=k} a_{k,\kappa} \Big(\tfrac{n}{2}\Big)_\kappa C_\kappa(\Ib_m) dt .\label{eqn:full_exact_solution}
\end{align}
% %
% The argument double summation can actually be represented as a Hypergeometric function,
% %
% \begin{lemma}
% Let $p \ge 0$ and $q \ge 0$ be integers, and let $X$ be an $m \times m$ complex symmetric matrix. Then the hyper-geometric function of a matrix argument $X$ and a parameter $\alpha > 0$, in the case of real symmetric matrices $\alpha = 2$, which returns Zonal polynomials rather than the generalized Pochhammer symbol, which is the case in this system and will be excluded, is defined as,
% \begin{align}
%   \,_pF_q(a_1, \cdots, a_p; b_1, \cdots, b_q; X) = \nonumber \\ \sum _{k=0}^{\infty }\sum _{\kappa \vdash k}{\frac {1}{k!}}\cdot {\frac {(a_{1})_{\kappa }\cdots (a_{p})_{\kappa }}{(b_{1})_{\kappa }\cdots (b_{q})_{\kappa }}}\cdot C_{\kappa }(X),
% \end{align}
% where $\kappa \vdash k$ means $\kappa$ is a partition of $k$.
% \end{lemma}
% %
% as such, we can rewrite the Equation \eqref{eqn:full_exact_solution} in terms of the hypergeometric function of a matrix argument,
% %
% \begin{align}
%   \Var{y}  &= \gamma^2 - \frac{\gamma^2 \alpha}{m} \int_{0}^{\infty} e^{-\alpha t} \,_1F_0\left(\frac{n}{2}, -2 \sigma^2 t \right) dt
% \end{align}
% %
\subsection{Series Expansion}

Given that we do not need the exact computation of the variance, we can evaluate the matrix exponential expansion for a small number of known Wishart moments, $k = 0, 1, 2$,
\begin{itemize}
    \item $k = 0$: $\E{\tr(S^0)} = \E{\tr(\Ib_m)} = m$.
    \item $k = 1$: $\E{\tr(S^1)} = \E{\tr(S)}$, for the first moments we know a Wishart distribution that $\E{S} = n \Sigma = n \sigma^2 \Ib_m$, which in turn gives, 
    \begin{align} \label{eqn:explicit_moment_1}
        \E{\tr(S)} = n m \sigma^2.
    \end{align}
    \item $k = 2$: $\E{\tr(S^2)}$, for the second moment of the central Wishart distribution \cite[Eqn. 4]{1982SamplesDistributions},
    \begin{align*}
        \Cov{S_{ij}, S_{kl}} = n(\sigma_{ik} \sigma_{jl} + \sigma_{il} \sigma_{jk}),
    \end{align*}
    here, given that $\Sigma = \sigma^{2} \Ib_m$, we thus have that $\E{S_{ij}} = n \sigma^2 \delta_{ij}$ and $\Cov{S_{ij}, S_{kl}} = n \sigma^4 (\delta_{ik} \delta_{jl} + \delta_{il} \delta_{jk})$, this implies that,
    \begin{align*}
        \E{S_{ij}S_{kl}} = \Cov{S_{ij}, S_{kl}} + \E{S_{ij}}\E{kl},
    \end{align*}
    we compute,
    \begin{align}
        \E{\tr(S^2)} &= \sum_{i, j} S_{ij}S_{ji}, \nonumber  \\
        &= \sum_{i, j} (n \sigma^4 (\sigma_{ij} \sigma_{ji} + \sigma_{ii} \sigma_{jj}) + n^2 \sigma^4 (\delta_{ij} \delta_{ji}), \nonumber \\
        &= n \sigma^4 (m + m^2) + n^2 \sigma^4 m, \nonumber \\
        &= \sigma^4 nm (m + 1 + n), \label{eqn:explicit_moment_2}
    \end{align}
\end{itemize}
\subsection{Series expansion algorithm} \label{sec:series_expansion}

While the lower moments up to $k = 2$ were derived manually, we need algorithms to handle higher orders. To do this, we make use of Isserlis \cite{Isserlis1918OnVariables} or Wick theorem \cite{Wick1950TheMatrix}. Let $W_0 \in \R^{m \times n}$ have IID $W_{ij} \sim \Nl(0, \sigma^2)$ and $S = W_0 W_0^\top$, we can expand the trace,
\begin{align*}
    \tr(S^k) = \sum_{i_1, \cdots, i_k = 1}^m \sum_{i_1, \cdots, i_k = 1}^n \prod_{t = 1}^k W_{i_t, j_t} W_{i_{t + 1}, j_{t}}
\end{align*}
where $i_{k + 1} := i_{1}$. This formulation implies that there are $2k$ Gaussian factors. By the Wick theorem \cite{Wick1950TheMatrix}, the expectation of a product of $2k$ zero-mean Gaussian is equal to the sum over all perfect matchings, which cover every vertex of the graph \cite{Godsil2001AlgebraicTheory}, of the $2k$ positions of the product of the pairwise covariances. For our variables,
\begin{align*}
    \E{W_{a,p}W_{b, q}} = \sigma^2 \delta_{ab} \delta_{pq}
\end{align*}
Hence, any matched pair imposes equality between the corresponding row and column indices. Concretely, a matching $m \in \mathcal{M}_{2k}$ induces a system of equalities among the row indices $\{i_t\}$ and column indices $\{j_t\}$.

The combinatorial effect of each constraint is that,
\begin{itemize}
    \item The factor $\sigma^{2k}$ is derived from $k$ covariances,
    \item a factor $n^{a}$, where $a$ represents the number of independent connected components from the column indices of the matching,
    \item a factor $m^{b}$, where $a$ represents the number of independent connected components from the row indices of the matching,
\end{itemize}
thus each matching $m \in \mathcal{M}_{2k}$ contributes,
\begin{align*}
    \sigma^{2k} n^{a (m)} m^{a(m)}
\end{align*}
and thus summing over all perfect matchings, we have that,
\begin{align*}
    \E{\tr(S^k)} = \sigma^{2k} \sum_{m \in \mathcal{M}_{2k}} n^{a (m)} m^{a(m)},
\end{align*}
where $\mathcal{M}_{2k}$ is the perfect matchings of the $2k$ positions. It should be noted that $|\mathcal{M}_{2k}| = (2k - 1)!!$, which grows super-exponentially, making this computation practical only for small $k$. Grouping identical monomials returns the form,
\begin{align*}
    \E{\tr(S^k)} = \sigma^{2k} \sum_{a, b \ge 0} C_{a,b}^{(k)} n^a m^b,
\end{align*}
where $C_{a,b}\in \N_{+} $ represents the integer monomial counts for the degrees $a$ and $b$. The higher moments can thus be computed and are derived below. Each coefficient $ C_{a,b}$ is represented in the matrix associated for vectors $\boldsymbol{n} = (n^k, \cdots, n)^\top$ and $\boldsymbol{m} = (m, \cdots, m^k)^\top$, where each element of the matrix $\boldsymbol{n} \boldsymbol{m}^\top$ represents the coefficients $C^{(k)}$. This is an asymptotic series, not a convergent one.
\begin{itemize}
    \item $k =1, 2$, which matches Equations \eqref{eqn:explicit_moment_1} and \eqref{eqn:explicit_moment_2} respectively,
    \begin{align*}
        \begin{bmatrix}1\end{bmatrix}, \qquad
        \begin{bmatrix}1 & 0\\1 & 1\end{bmatrix}
    \end{align*}
    \item $k =3, 4$ and $k= 5$ respectively computed using Wick's theorem,
    \begin{align*}
    \setlength\arraycolsep{3pt}
        \begin{bmatrix}1 & 0 & 0\\3 & 3 & 0\\4 & 3 & 1\end{bmatrix}, \quad
        \begin{bmatrix}1 & 0 & 0 & 0\\6 & 6 & 0 & 0\\21 & 17 & 6 & 0\\20 & 21 & 6 & 1\end{bmatrix},  \quad
        \begin{bmatrix}1 & 0 & 0 & 0 & 0\\10 & 10 & 0 & 0 & 0\\65 & 55 & 20 & 0 & 0\\160 & 175 & 55 & 10 & 0\\148 & 160 & 65 & 10 & 1\end{bmatrix}
    \end{align*}
    \item Higher order moments up to $k=10$ are available in Appendix \ref{sec:moments}.
\end{itemize}
To demonstrate that the estimated layer variance for small $\sigma^2$ is equal to the actual layer variance propagation, the actual layer variance propagation was sampled using Monte Carlo estimation for the mean. It was compared with the computed variance estimate based on the $k = \{1, \cdots, 10\}$ moments derived above. The errors are graphed in Figures \ref{fig:var_1}, \ref{fig:var_10} and \ref{fig:var_100}. 
\begin{figure}[ht!]
    \centering
    \includegraphics[width=1\linewidth]{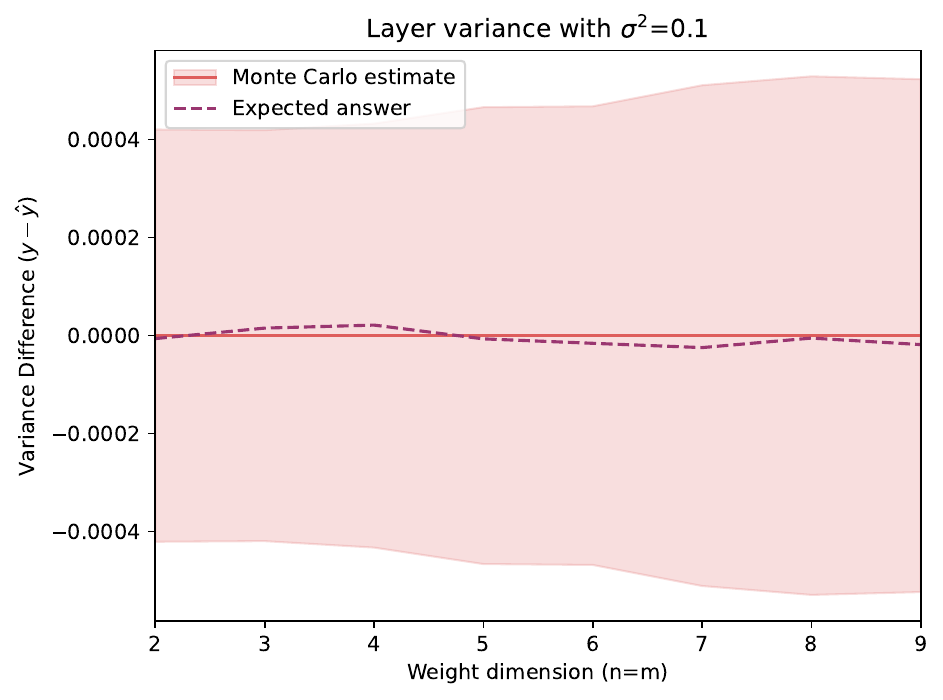}
    \caption{Variance difference estimation for weight parameterization sizes from 2 to 9}
        \label{fig:var_1}
\end{figure}
\begin{figure}[ht!]
        \centering
        \includegraphics[width=1\linewidth]{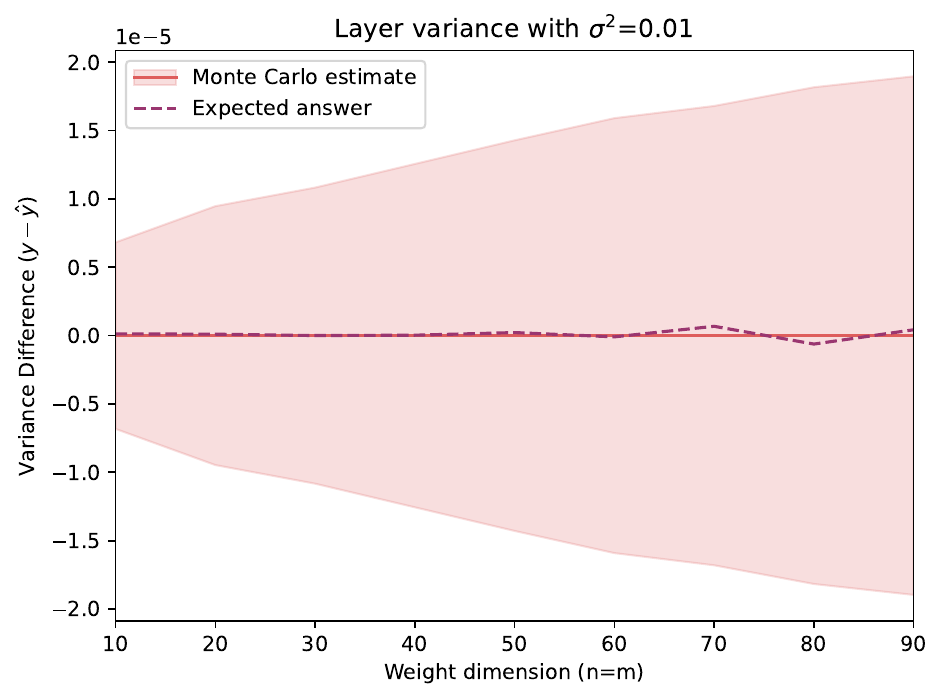}
        \caption{Variance difference estimation for weight parameterization sizes from 10 to 90}
        \label{fig:var_10}
    \end{figure}
\begin{figure}[ht!]
        \centering
        \includegraphics[width=1\linewidth]{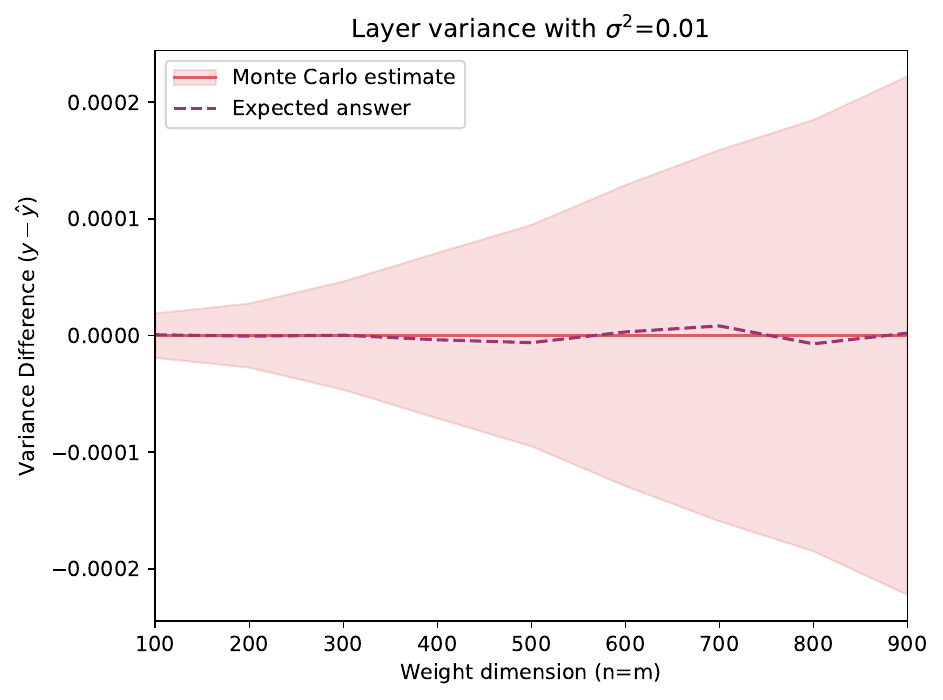}
        \caption{Variance difference estimation for weight parameterization sizes from 100 to 900}
        \label{fig:var_100}
    \end{figure}
Based on the empirical errors, we can confirm that the truncated moment-variance estimate is accurate and adequately represents the parametrization's actual variance for small $\sigma^2$. The error bars in the figures represent one standard deviation.

\subsection{Variance scaling}

The Laplace integration's matrix exponential term $e^{-St}$ is lower bounded by the smallest eigenvalue, $\lambda_{\min}$, of the matrix $S$. Given that $S \sim \Wl_m(n, \sigma^2 \Ib_m)$ is equivalent to $\sigma^2 S \sim\Wl_m(n, \Ib_m)$ this implies that $\lambda_{\min} \propto \sigma^2$ and consequently,
\begin{align*}
    \lim_{\sigma^2 \to \infty} \E{ \trace (\alpha \Ib_m + S)^{-1}} = 0
\end{align*}
which implies that the marginal variance,
\begin{align*}
   \lim_{\sigma^2 \to \infty} \Var{y} &= \gamma^2 - \frac{\alpha \gamma^2}{m} \lim_{\sigma^2 \to \infty}\E{ \trace (\alpha \Ib_m + S)^{-1}} = \gamma^2
\end{align*}
Thus, to achieve a unit-variance layer, we have to ensure that $\sigma^2$ is very large. This is illustrated through Figures \ref{fig:variance_std_linear} and \ref{fig:variance_std_nonlinear}, using $\alpha = \gamma = 1$.
\begin{remark}[Gradient Trade-off] \label{rm:gradient}
    As the initialization variance of $W_0$ increases to satisfy $\Var{y} \approx 1$, the expectation of the singular values of $W_0$ grows large. In the limit of large $W_0$, the term $(\alpha + W_{0}^\top W_0)$ dominates, and the normalized matrix saturates towards the boundary of the Lipschitz manifold. This saturation causes the gradients $\frac{\partial M}{\partial W_{0}}$ to vanish, potentially hindering optimization despite improved forward signal propagation.
\end{remark}
\begin{figure}[ht!]
    \centering
    \includegraphics[width=1\linewidth]{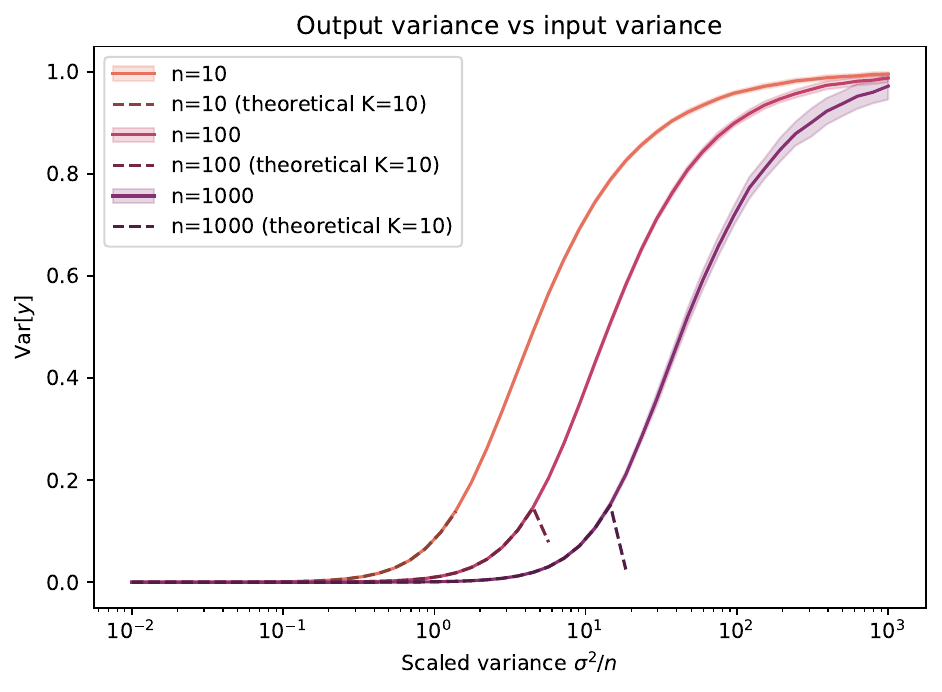}
    \caption{Relationship between the output variance and the normalized input variance using $n$ scaling at multiple $W_0 \in \R^{n \times n}$ dimensions with $\alpha = \gamma = 1$}
    \label{fig:variance_std_linear}
\end{figure}
\begin{figure}[ht!]
    \centering
    \includegraphics[width=1\linewidth]{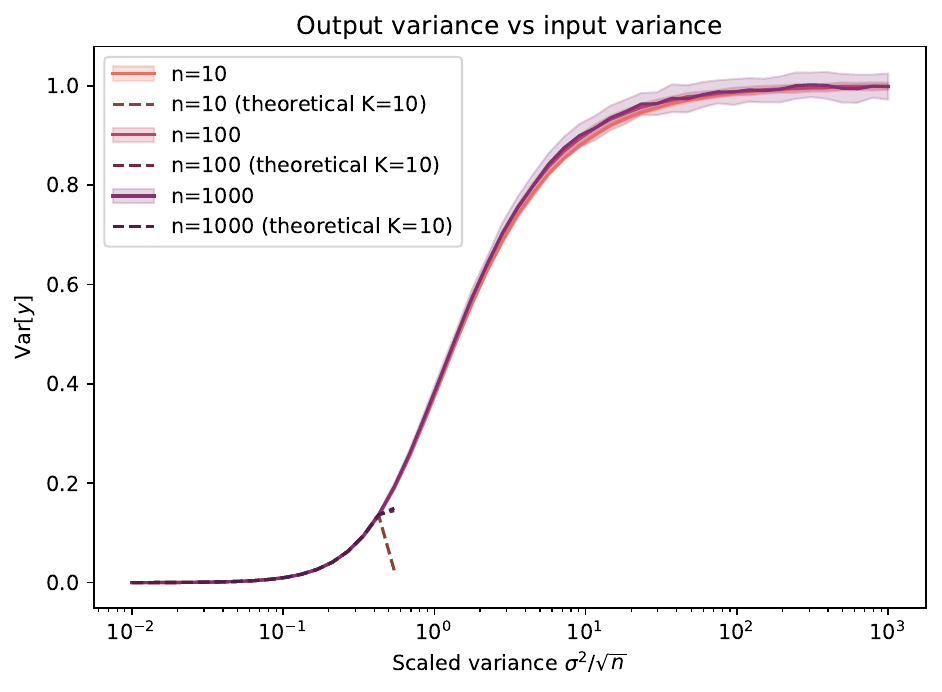}
    \caption{Relationship between the output variance and the normalized input variance using $\sqrt{n}$ scaling at multiple $W_0 \in \R^{n \times n}$ dimensions with $\alpha = \gamma = 1$}
    \label{fig:variance_std_nonlinear}
\end{figure}
The theoretical values plotted are based on the truncated expectations that were derived in Section \ref{sec:series_expansion}. The figures thus also illustrate the region where the approximation in Section \ref{sec:series_expansion} is valid for $\sigma^2$. The Figure \ref{fig:variance_std_nonlinear} also illustrates that the output variance scales with $\sqrt{n}$. Thus, to obtain an output variance close to one, $10 / \sqrt{n}$ would be an appropriate scaling factor. The standard initialization factor would be $1 / \sqrt{n}$, which would result in an output variance of $\approx 0.41$, while the new variance would bring it close to $\approx 0.9$.

The proof that the output variance is dependent on $\sigma^2$ and that it is theoretically possible to achieve near unit output is a huge advantage compared to previously developed Lipschitz layer parameterizations such as SLL \cite{Araujo2023} or AOL \cite{Prach2022}; where, \cite{Juston20251-LipschitzProblem} demonstrates that for the parameterization employed previously the output variance is independent on $\sigma^2$ meaning that the only control one could employ would be to change the weight dimensions, which in practice is not practical.

\subsection{Backward-propagation}

Backward propagation was very similar to forward propagation. Instead of the previous forward-propagation equation, the layer was rearranged to:
\begin{align}
    \Delta \boldsymbol{x_l} &= \Tilde{W}_l  \Delta \boldsymbol{y_l} . 
\end{align}
Where $\Delta \boldsymbol{x_l}$ and $\Delta \boldsymbol{y_l}$ denote the gradients $\frac{\partial \mathcal{L}}{\partial \boldsymbol{x}}$ and $\frac{\partial \mathcal{L}}{\partial \boldsymbol{y}}$  respectively. This results in the gradient variance from Equation \eqref{eqn:full_exact_solution}, but with gradient-based dimension parameters and no bias term.

\section{Experiments}

\subsection{Data set}

To validate, we use the Higgs data set generated by Monte Carlo simulations. The first 21 features are kinematic properties measured by particle detectors in the particle accelerator. The last seven features are high-level features derived by physicists to help discriminate between Higgs Boson particles and non-Higgs particles \cite{Baldi2014SearchingLearning}.

We used the sub-sampled dataset with 100k training data points and 20k test data points. We use stratified k-fold cross-validation with a size of 5, a batch size of 64, an initial learning rate of $1e-3$ with a Cosine Annealing learning rate scheduler, which reduces to a learning rate of $1e-4$, 

\subsection{Training Setup}
The training was performed across cross-validation folds, with weight initialization scales $\alpha = \{1,2,4,8,10\}$ and network depths $L = \{5, 10, 15,20,25,30\}$, and finally with optimizers SGD and AdamW \cite{Loshchilov2017DecoupledRegularization} with default parameters. This generated 300 different training configurations.

We tested the $LDL^\top$ linear network \cite{Juston2025LDLTConstruction}, with a constant width of 512 neurons and ELU activation functions \cite{Clevert2015} between each layer. Using 4 A6000 RTX NVIDIA with 48GB of VRAM, using 2 AMD EPYC 7713 64-Core Processors for a cumulative of 328.637 GPU compute hours.

\subsection{Results}

We log the validation AUC during network training and illustrate it with respect to the optimizer and the network depth or initialization scale, respectively, in Figure \ref{fig:val_auc}.
%
% \begin{figure}[ht!]
%     \centering
%     \includegraphics[width=1\linewidth]{Figures/Training/epoch/pdf/val_auc_depth.pdf}
%     \caption{Comparison of the validation AUC training dynamics with respect to the $LDL\top$ network's depth between the SGD and the AdamW optimizers. The confidence intervals are computed using bootstrapping of 1000 samples. When observing the test AUC during training, we observe that there does not seem to be a clear distinction in performance between layer depth, with the training using AdamW optimizer consistently performing better than SGD.}
%     \label{fig:val_auc_depth}
% \end{figure}
%
% \begin{figure}[ht!]
%     \centering
%     \includegraphics[width=1\linewidth]{Figures/Training/epoch/pdf/val_auc_scale.pdf}
%     \caption{Comparison of the validation AUC training dynamics with respect to the initialization scale between the SGD and the AdamW optimizers. The confidence intervals are computed using bootstrapping of 1000 samples. When observing the test AUC during training, we see a mixed message: the initialization scale and the test AUC are both increasing. The AdamW optimizer consistently outperforms SGD and shows no relationship with the initialization scale; however, SGD shows a clear response: the smaller the scale, the better the test AUC.}
%     \label{fig:val_auc_scale}
% \end{figure}
%
\begin{figure}
    \centering
\begin{subfigure}[b]{1\linewidth}
    \centering
    \includegraphics[width=1\linewidth]{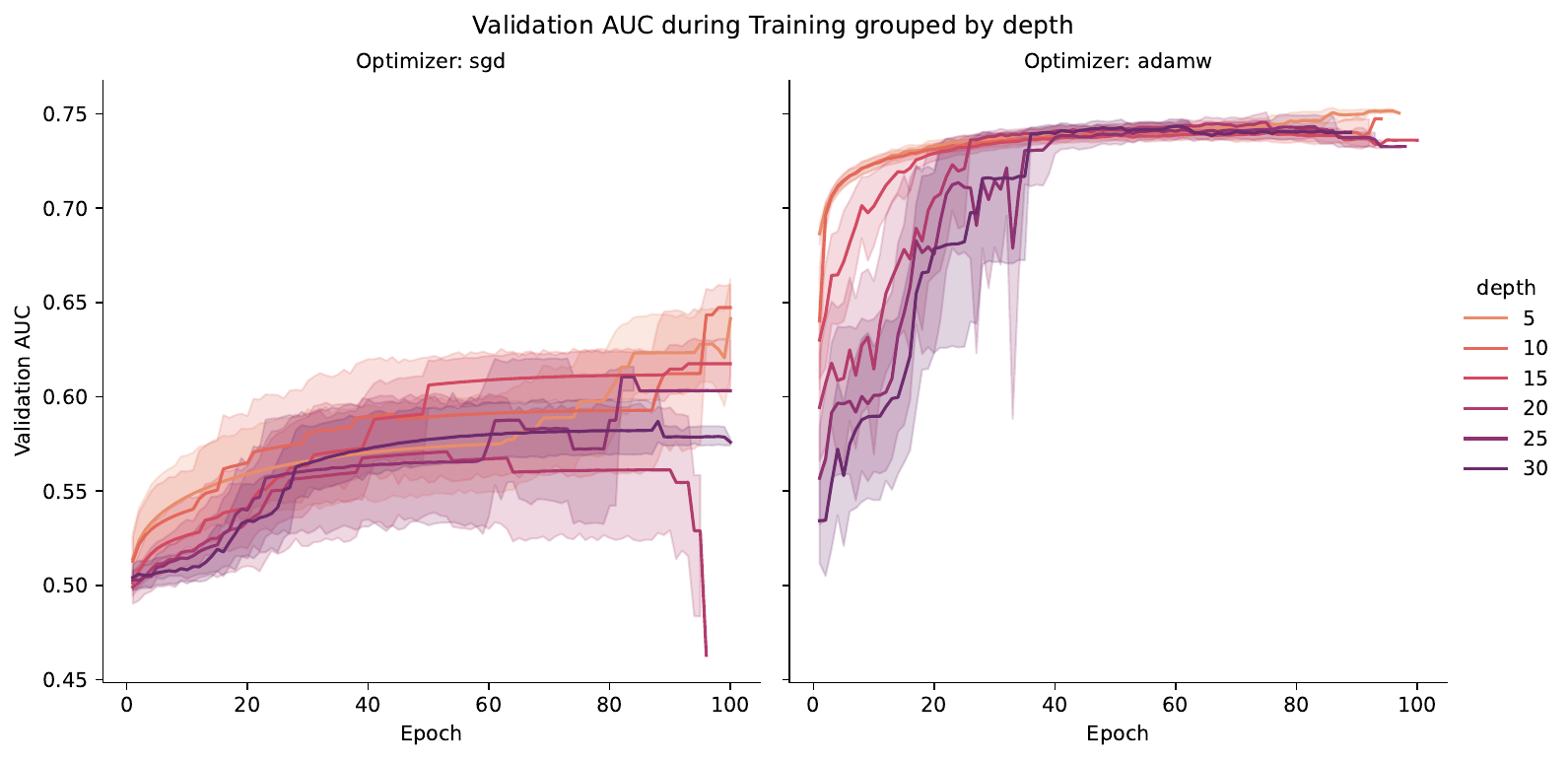}
    \caption{When observing the test AUC during training, we observe that there does not seem to be a clear distinction in performance between layer depth, with the training using the AdamW optimizer consistently performing better than SGD.}
    \label{fig:val_auc_depth}
\end{subfigure}
\begin{subfigure}[b]{1\linewidth}
    \centering
    \includegraphics[width=1\linewidth]{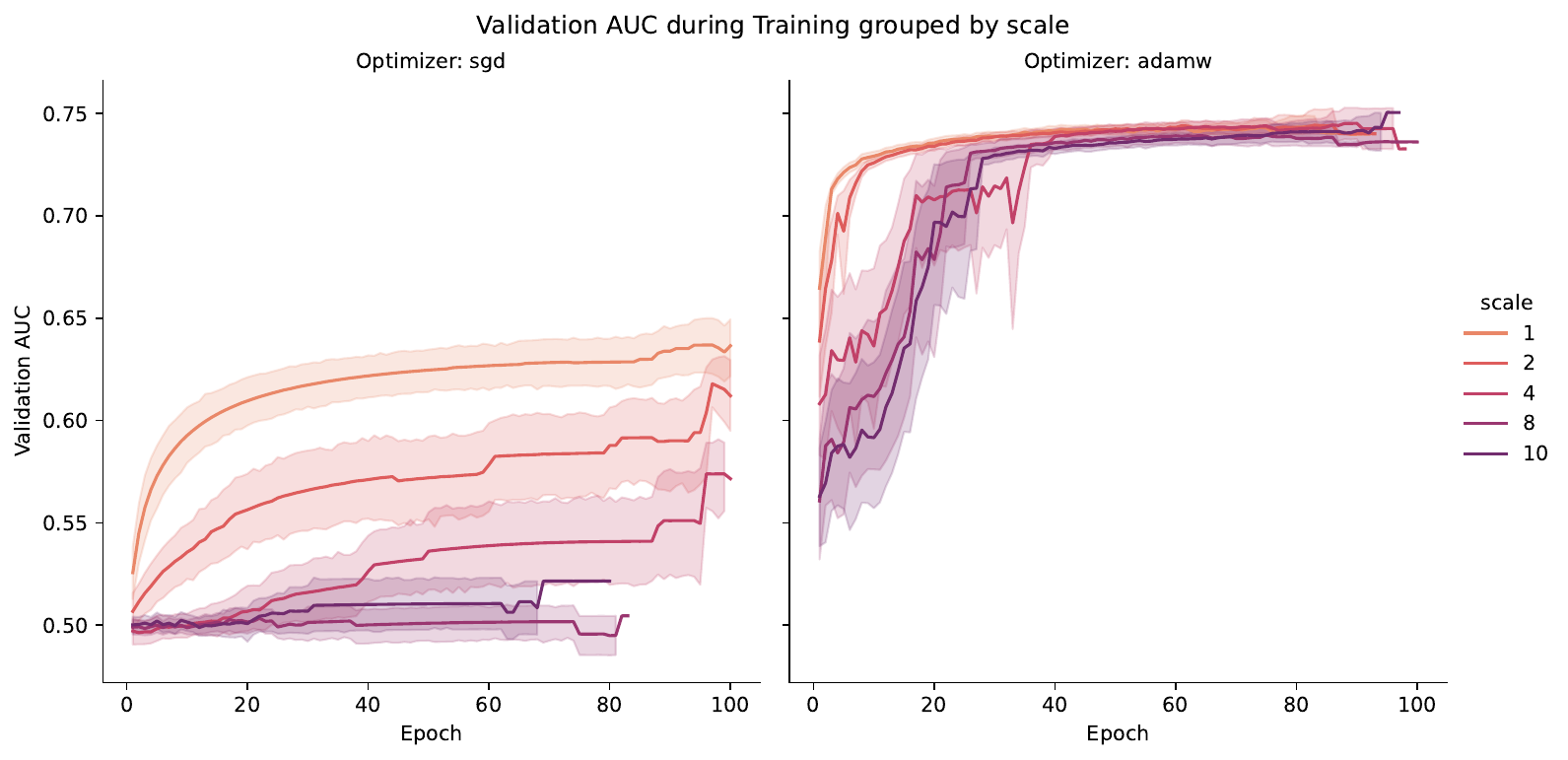}
    \caption{When observing the test AUC during training, we see a mixed message: the initialization scale and the test AUC are both increasing. The AdamW optimizer consistently outperforms SGD and shows no relationship with the initialization scale; however, SGD shows a clear response: the smaller the scale, the better the test AUC.}
    \label{fig:val_auc_scale}
\end{subfigure}
    \caption{Comparison of the validation AUC training dynamics with respect to the network's depth and initialization scale, respectively, between the SGD and the AdamW optimizers. The confidence intervals are computed using bootstrapping of 1000 samples. }
    \label{fig:val_auc}
\end{figure}
From the validation AUC plots, we notice that, relative to the depth hyperparameter, there is no expected decay, with the larger networks showing little to no decrease in accuracy as depth increases. In addition, in contrast to the theory above, we would also assume that the larger the initialization scale $\alpha$, the larger the AUC should be; however, as illustrated in Figure \ref{fig:val_auc_scale}, that is not the case. In training with SGD, we observe the opposite relationship: the larger the initialization scale, the lower the expected validation AUC. In contrast, during AdamW optimizer training, there is not only no relationship between the initialization scheme's scale and training, but it also seems to follow the same trend as the network's depth. The invariance to the network's initialization makes sense due to Adam's family normalization term in its parameter update step \cite{Loshchilov2017DecoupledRegularization, Kingma2014Adam:Optimization}, 
\begin{align*}
    \theta_t \gets \theta_t - \gamma \frac{\hat{m}_t}{ \sqrt{\hat{v}_t} + \epsilon},
\end{align*}
where $\hat{m}_t \propto g_t$ and $\hat{v}_t \propto g_t^2$ represent the gradient's first and second moments, respectively. 
% This effectively normalizes the parameter's scale. This thus explains the scale invariance of AdamW's training.
As noted in Remark \ref{rm:gradient}, higher initialization scale $(\sigma^2 \propto 10 / \sqrt{n}$) saturate the normalization term, resulting in gradients $\nabla_{W_0} \mathcal{L}$ whose magnitudes scale inversely with $\sigma^2$. SDG, which updates $W_{t + 1} \gets W_{t + 1} - \eta \nabla \mathcal{L}$, cannot overcome these smaller gradients, leading to the ``lazy regime" failure observed in Figure \ref{fig:hessian_vs_regime_correlation}. In contrast, AdamW rescales the update by the inverse of the gradient's second moment (approximating the Hessian diagonal), effectively canceling out the scaling factor introduced by the initialization. This allows AdamW to traverse the ``rich regime" regardless of the initialization scale.
\begin{figure}[ht!]
    \centering
    \includegraphics[width=1\linewidth]{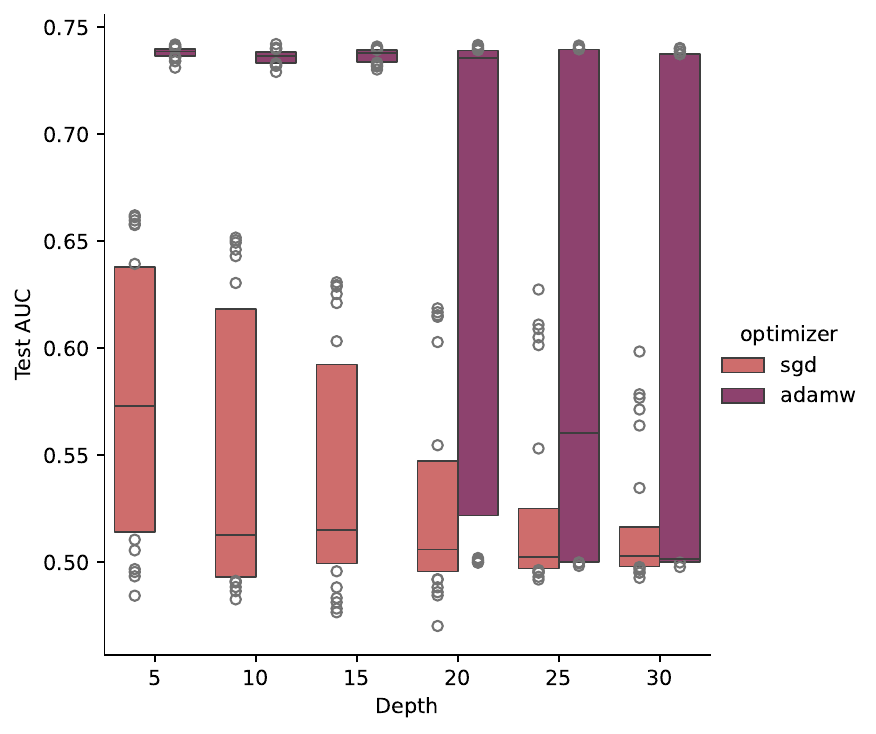}
    \caption{Comparison plots between two optimizers, AdamW and SGD, of the final test data set's AUC grouped by the network's depth for training the $LDL^\top$ network on the Higgs Boson classification dataset. As illustrated for both optimizers, performance decreases with increasing network depth.}
    \label{fig:test_auc_depth_box}
\end{figure}
\begin{figure}[ht!]
    \centering
    \includegraphics[width=1\linewidth]{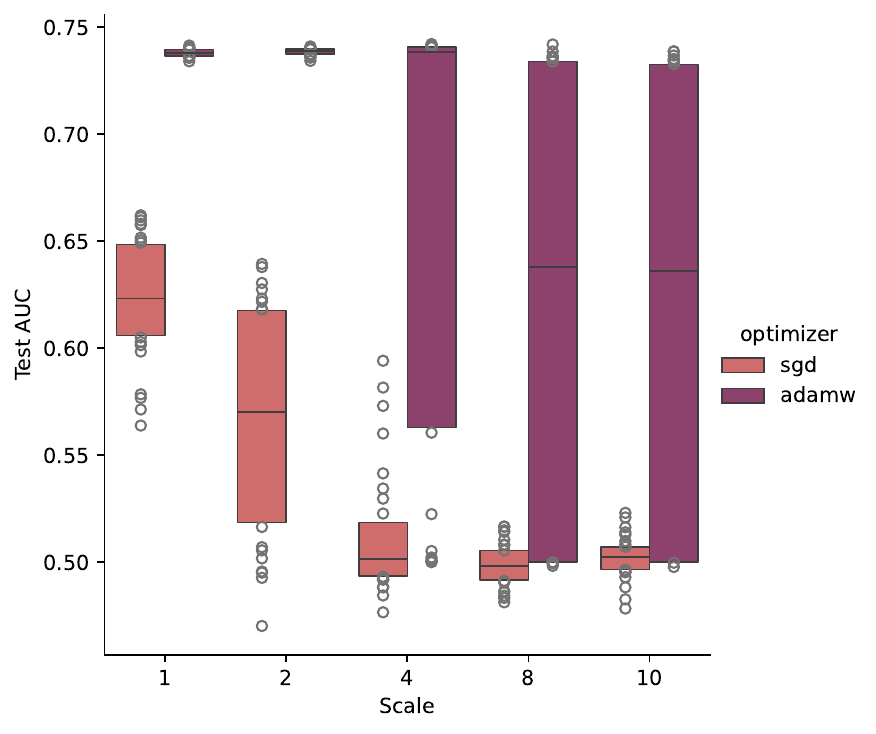}
    \caption{Comparison plots between two optimizers, AdamW and SGD, of the final test data set's AUC grouped by the network's scale for training the $LDL^\top$ network on the Higgs Boson classification dataset. As illustrated for both optimizers, the performance decreases as the initialization scale increases.}
    \label{fig:test_auc_scale_box}
\end{figure}
When observing the test AUC box plots in Figures \ref{fig:test_auc_depth_box} and \ref{fig:test_auc_scale_box} with respect to the network's depth and scale respectively, we notice the opposite trend, where as the network depth increase, the performance decrease for both the SGD and the AdamW optimizers; however, similarly to the validation AUC, the larger the initialization scale, the worse the test AUC ended up being. This negative trend is reflected in both the AdamW and SGD optimizers' final test AUC. This is especially prevalent in network training using the SGD optimizer, which effectively works no better than random chance, with an AUC of 0.5.

To try to explain this, we look at the final Hessian trace of the loss, 
\begin{align*}
    \sum_{i,j}\frac{\partial^2 \mathcal{L}}{\partial W_{ij}^2},
\end{align*}
with respect to the maximum Parameter Movement, the distance relative to the initial parameter value,
\begin{align*}
        \arg \max_\tau \frac{\|W_\tau - W_0 \|}{\|W_0\|}
\end{align*}
illustrated in Figure \ref{fig:hessian_vs_regime_correlation}, which illustrates that the AdamW optimizer operates in the rich regime during training, where the parameters change a lot, while the SGD optimizer usually operates in the smaller lazy regime, the regime where the network weights do not change much from their initialization. We can also notice that when the network achieved a larger Hessian trace and a sharper minimum, the networks consistently achieved larger test AUC, as indicated by the smooth color transition illustrated in Figure \ref{fig:hessian_vs_regime_correlation}.  
\begin{figure}[ht!]
    \centering
    \includegraphics[width=1\linewidth]{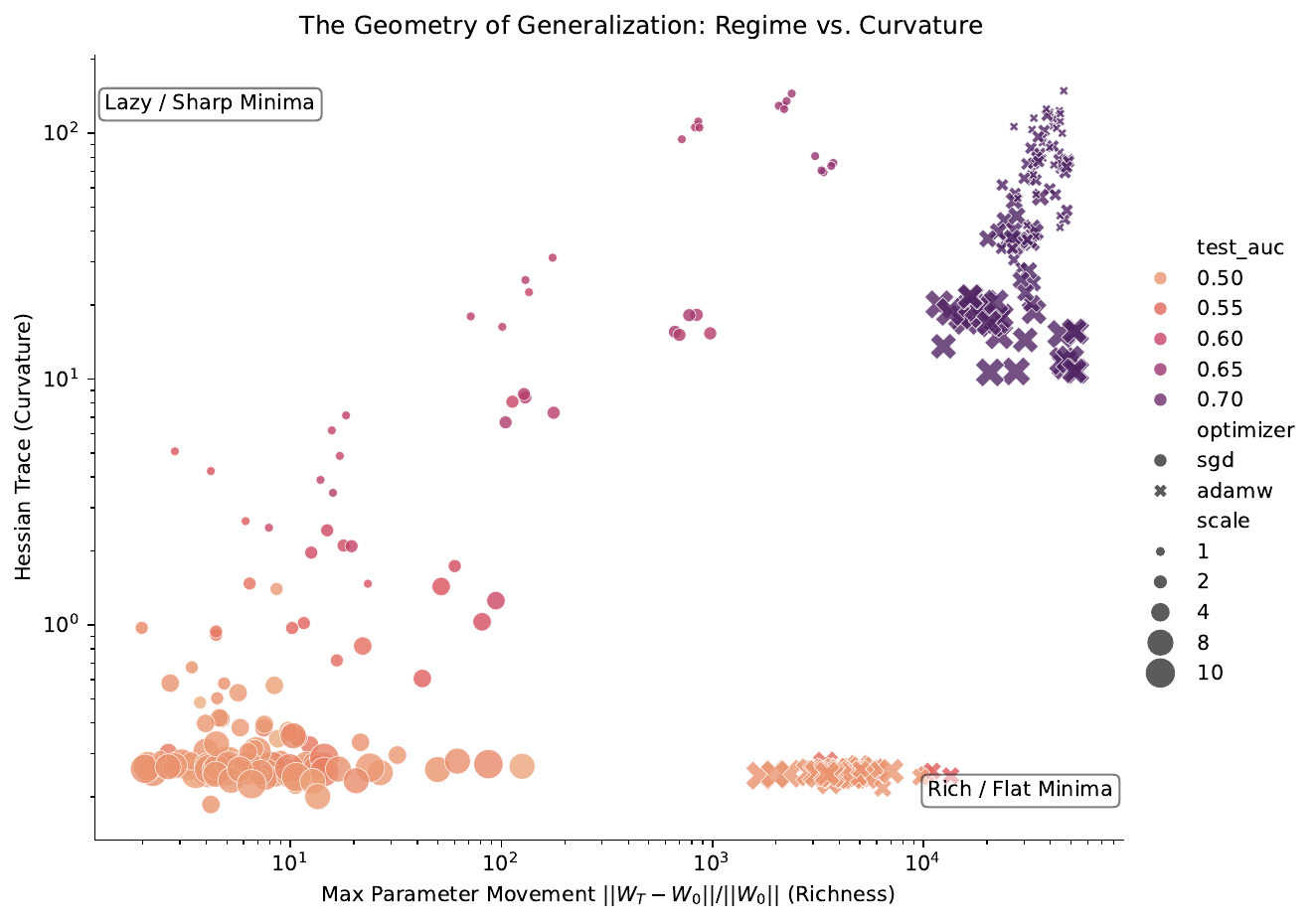}
    \caption{Hessian trace relative to the parameter movement with respect to the test AUC, the optimizer, and the initialization scale. AdamW operates in the richness regime, while SGD operates in a lazier regime. When runs operated in the sharp-rich regimes, the highest test AUCs were achieved.}
    \label{fig:hessian_vs_regime_correlation}
\end{figure}
We thus explore the factors that achieve this sharp Hessian trace, where Figure \ref{fig:hessian_trace_scale} demonstrates an apparent relationship between the Hessian's trace and the initialization scale. The larger the initialization scale, the smaller the Hessian trace, and vice versa for both optimizers. This thus implies that larger initialization scales lead the network to a flatter minimum, while smaller initialization scales lead the network to a more curved, sharper minimum.

Prior works has been worked on proposing that flat minima exhibit stronger generalization ability, due to their invariance to the parameter sensitivity, \cite{Hochreiter1997FlatMinima, JastrzEbskiFindingSGD, Liu2025ALoss}; It was also demonstrated that SGD naturally gravitates towards flatter regions, which could potentially explain why in Figure \ref{fig:hessian_trace_depth}, while AdamW becomes sharper as the network size increases, SGD does the opposite and becomes flatter. However, this relationship between flat local minima and generalization ability has been called into question by others, who show that the opposite holds \cite{Zhang2021WhyNetworks, Dinh2017SharpNets}, underscoring the need for further analysis. In this network, we notice that sharper minima seem to yield better test AUC.
%
% \begin{figure}[ht!]
%     \centering
%     \includegraphics[width=1\linewidth]{Figures/Training/epoch/pdf/hessian_trace_scale.pdf}
%     \caption{Comparison of the loss's Hessian trace training dynamics with respect to the initialization scale between the SGD and the AdamW optimizers. The confidence intervals are computed using bootstrapping of 1000 samples. Illustrating the relationship that the larger the initialization scale, the smaller the Hessian trace, thus implying a more stable local minima.}
%     \label{fig:hessian_trace_scale}
% \end{figure}
%
% \begin{figure}[ht!]
%     \centering
%     \includegraphics[width=1\linewidth]{Figures/Training/epoch/pdf/hessian_trace_depth.pdf}
%     \caption{Comparison of the loss's Hessian trace training dynamics with respect to the $LDL^\top$ network's depth between the SGD and the AdamW optimizers. The confidence intervals are computed using bootstrapping of 1000 samples. For SGD, as network depth increases, the Hessian trace decreases, whereas for AdamW it increases.}
%     \label{fig:hessian_trace_depth}
% \end{figure}
%
\begin{figure}
    \centering
    \begin{subfigure}[b]{1\linewidth}
    \centering
    \includegraphics[width=1\linewidth]{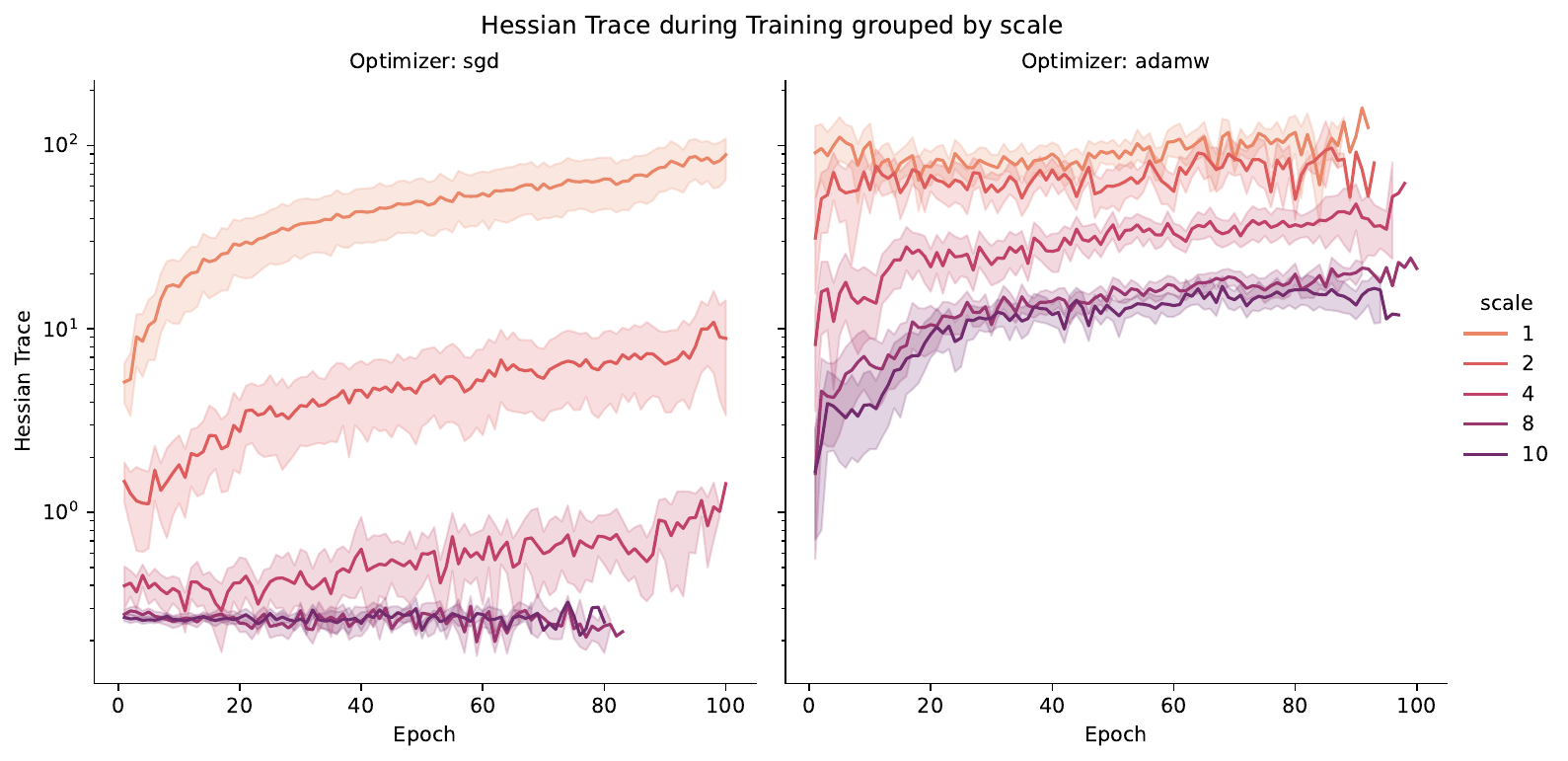}
    \caption{Illustrating the relationship that the larger the initialization scale, the smaller the Hessian trace, thus implying a more stable local minima.}
    \label{fig:hessian_trace_scale}        
    \end{subfigure}
    \begin{subfigure}[b]{1\linewidth}
    \centering
    \includegraphics[width=1\linewidth]{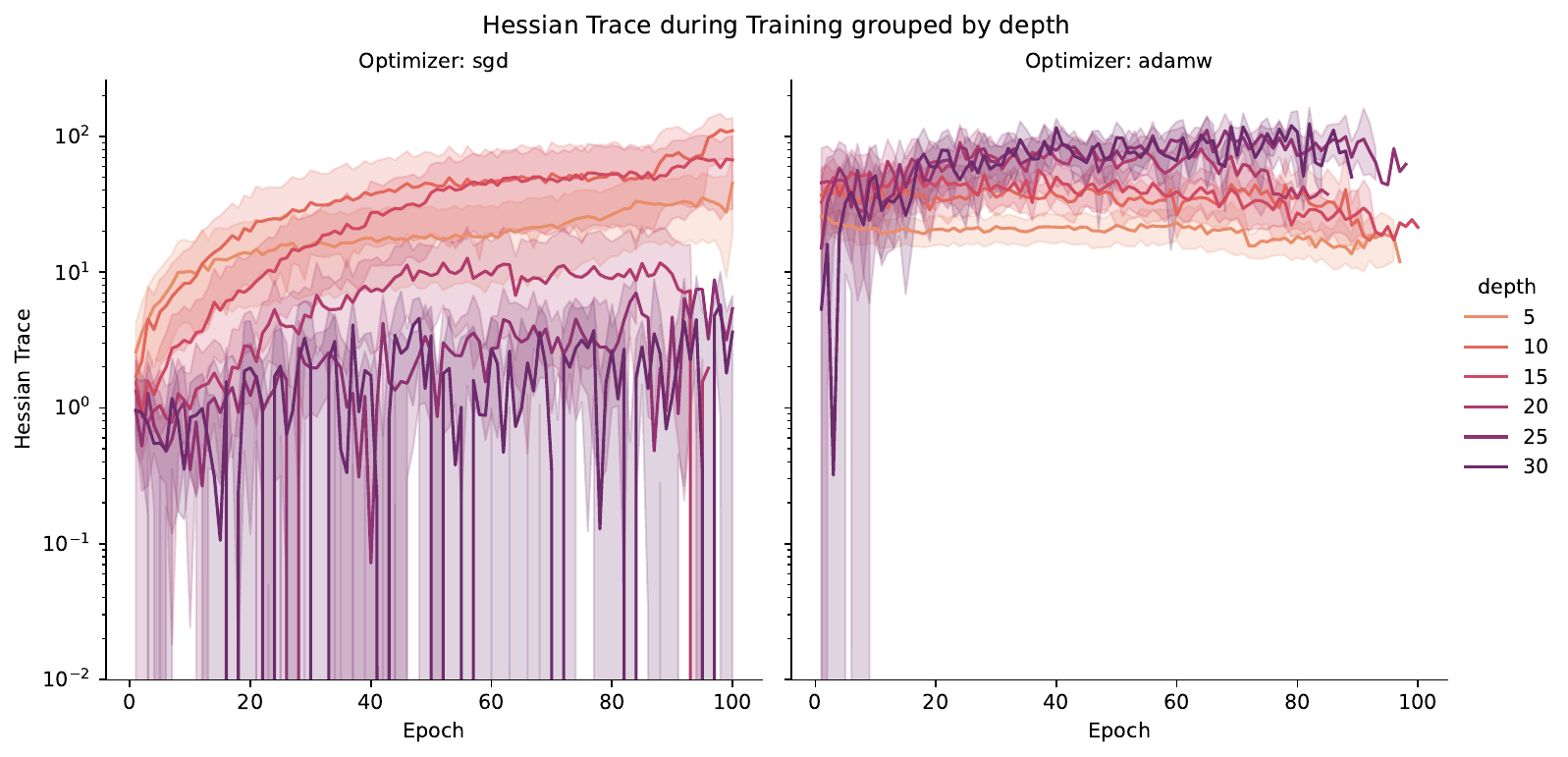}
    \caption{For SGD, as network depth increases, the Hessian trace decreases, whereas for AdamW it increases.}
    \label{fig:hessian_trace_depth}
    \end{subfigure}
    \caption{Comparison of the loss's Hessian trace training dynamics with respect to the initialization scale and the $LDL^\top$ network's depth, respectively, between the SGD and the AdamW optimizers. The confidence intervals are computed using bootstrapping of 1000 samples. }
    \label{fig:hessian_trace}
\end{figure}
\subsubsection{Parameter Movement}

Also, explore the motion of the network's movement during the training, expressed as,
\begin{align*}
    \frac{\|W_\tau - W_0\|}{W_0},
\end{align*}
the parameter's difference between its initialization. Figure \ref{fig:regime_lazy_vs_rich} illustrates a clear and consistent relationship between optimizers, such that the smaller the initialization scale, the larger the parameter movement. 
\begin{figure}[ht!]
    \centering
    \includegraphics[width=1\linewidth]{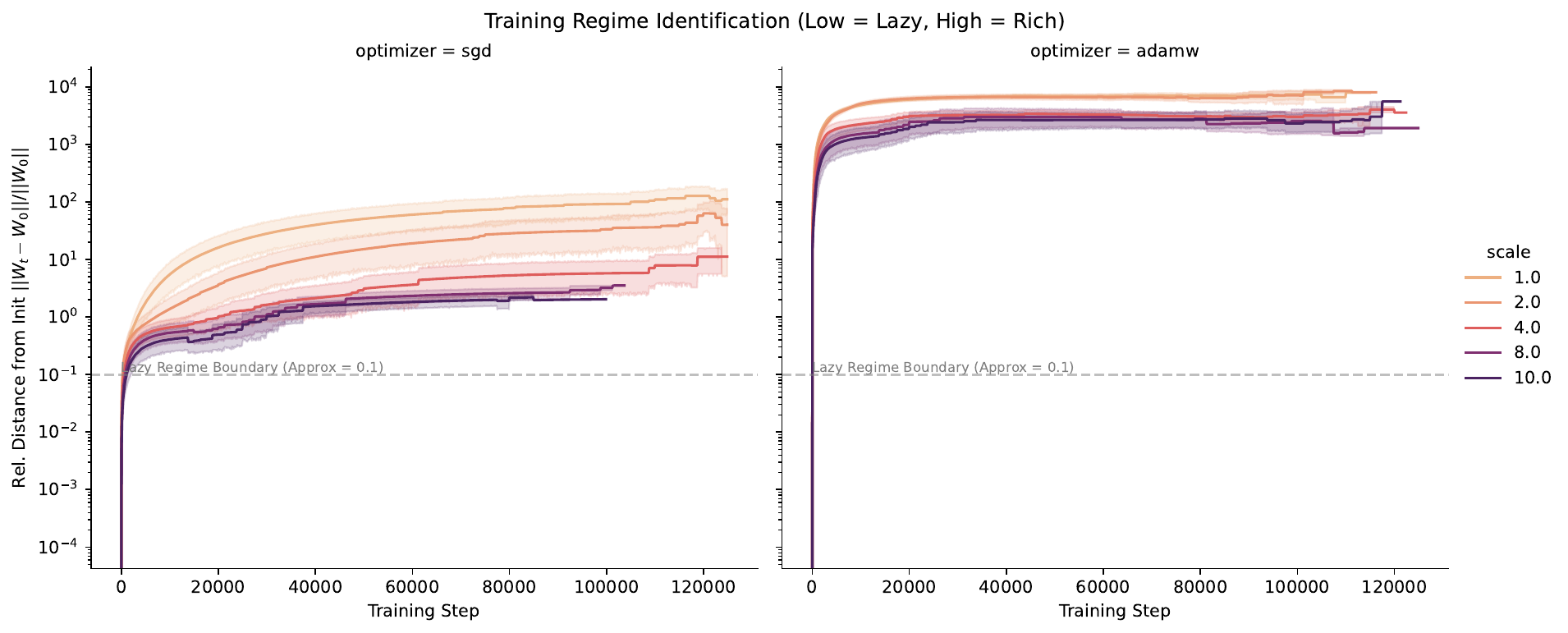}
    \caption{Hessian trace relative to the parameter movement with respect to the test AUC, the optimizer, and the initialization scale. A consistent relation shape between the initialization scale and parameter movement is observed. A smaller initialization scale leads to larger parameter movement.}
    \label{fig:regime_lazy_vs_rich}
\end{figure}
This helps indicate that the larger the initialization, the more powerful the normalization on the parameter weights is observed as in Remark \ref{rm:gradient}, given that the normalization in Equation \eqref{eq:normalization},
\begin{align*}
    M^\top M = \gamma^2 Q\diag(\frac{\mu_1}{\alpha + \mu_1}, \cdots,\frac{\mu_n}{\alpha + \mu_n})Q^\top
\end{align*}
where $Q$ is any orthogonal matrix. With a larger initialization scale, this implies that the probability that $\E{\|W_{ii}\|}$ will be larger, thus making the singular values $\mu_i$, representative of the matrix's parameters, closer to the plateauing regime of,
\begin{align*}
    \frac{\mu_i}{\alpha + \mu_i} \Rightarrow \frac{\partial}{\partial \mu_i} \frac{\mu_i}{\alpha + \mu_i}  = \frac{\alpha}{(\alpha + \mu_i )^2} \label{eq:relationship_grad}
\end{align*}
As such, the larger $\mu_i$ induced by the larger initialization scheme pushes the networks into a region where parameter gradients are smaller—making parameter training more difficult.

While the decaying gradients are assumed, when observing the parameter gradient norms throughout training, we do not observe the expected gradient collapse issue in Figure \ref{fig:parameter_grad_norm}. A relationship between the gradient norm scale and the depth and the initialization scale can be observed as expected from the Equation \eqref{eq:relationship_grad}, where Figure \ref{fig:parameter_grad_norm_scale}, shows that when the scale is one the initial gradients hover around $10^{-2}$, while when the scale is $10$ the gradients start at around $10^{-3}$, which when computing the gradient scaling ration from Equation \eqref{eq:relationship_grad} reflects the real work; however, at these gradient magnitudes the training should be drastically hindered and much smaller gradients would be expected, in the magnitudes of $10^{-7}$, to be called vanishingly small as illustrated by AdamW's ability to learn. This, in turn, leads to the still-unsolved question of why the initialization scheme did not work as expected.
%
% \begin{figure}[ht!]
%     \centering
%     \includegraphics[width=1\linewidth]{Figures/Training/all_metrics/pdf/parameter_grad_norm_depth.pdf}
%     \caption{.}
%     \label{fig:parameter_grad_norm_depth}
% \end{figure}
% %
% \begin{figure}[ht!]
%     \centering
%     \includegraphics[width=1\linewidth]{Figures/Training/all_metrics/pdf/parameter_grad_norm_scale.pdf}
%     \caption{.}
%     \label{fig:parameter_grad_norm_scale}
% \end{figure}
\begin{figure}
    \centering
    \begin{subfigure}[b]{1\linewidth}
    \centering
    \includegraphics[width=1\linewidth]{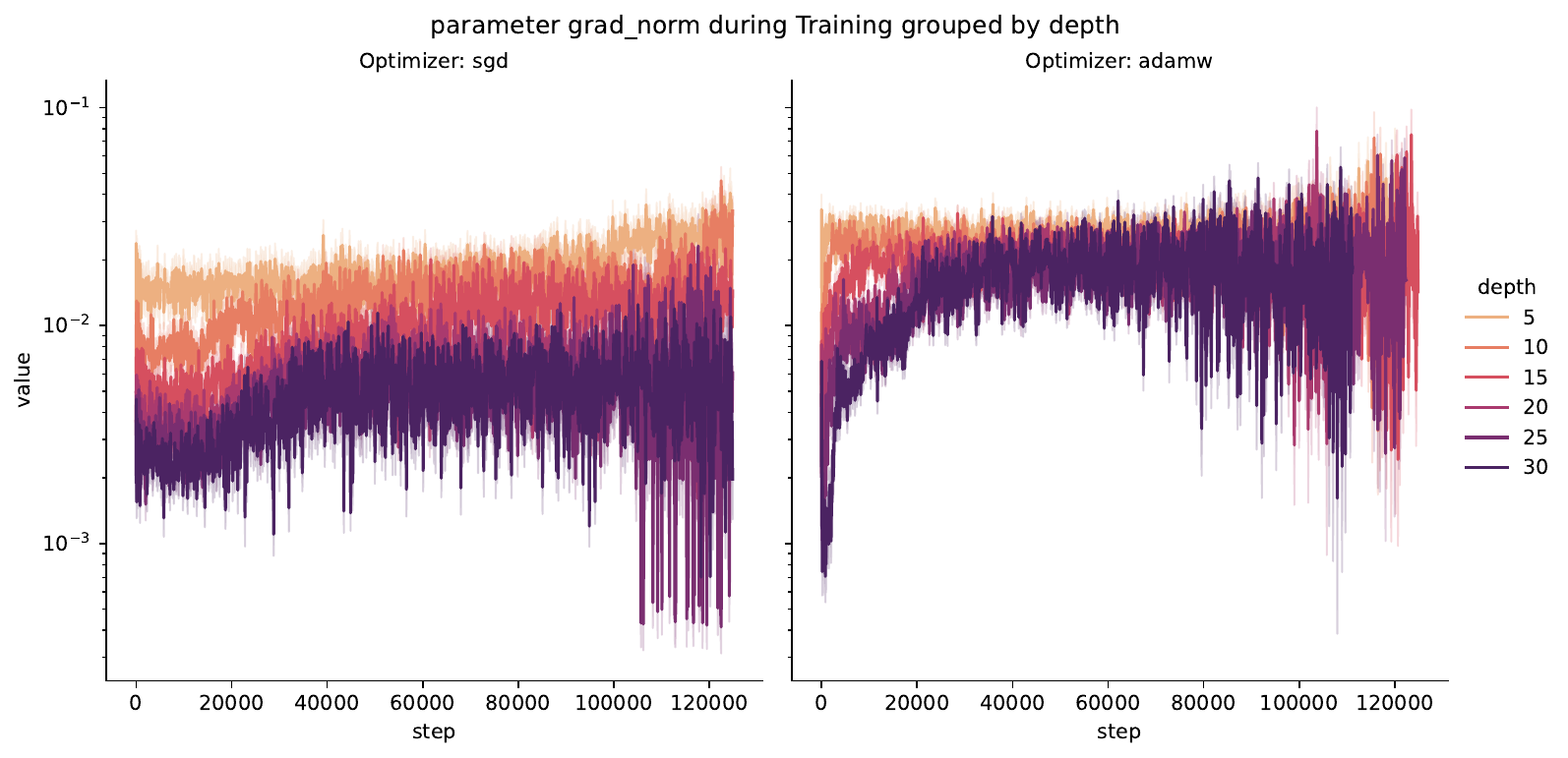}
    \caption{Demonstrates that the deeper the network, the smaller the parameter gradient norms.}
    \label{fig:parameter_grad_norm_depth}
\end{subfigure}
\begin{subfigure}[b]{1\linewidth}
    \centering
    \includegraphics[width=1\linewidth]{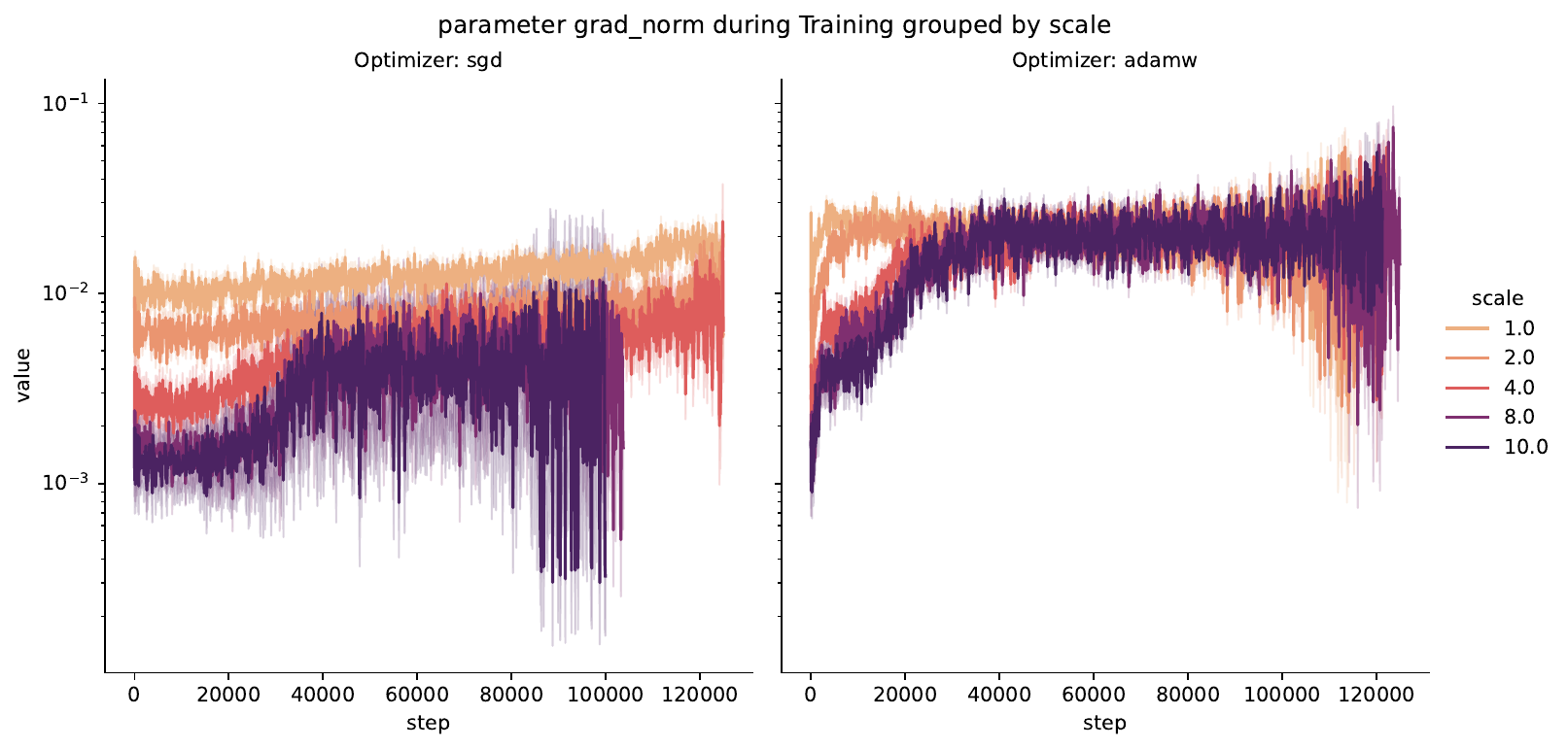}
    \caption{Demonstrates that the larger the initialization scale of the parameters, the smaller the parameter gradient norms.}
    \label{fig:parameter_grad_norm_scale}
\end{subfigure}
    \caption{Comparison of the parameter's gradient norm training dynamics with respect to the $LDL^\top$ network's depth and the initialization scale, respectively, between the SGD and the AdamW optimizers. The confidence intervals are computed using bootstrapping of 1000 samples.}
    \label{fig:parameter_grad_norm}
\end{figure}
\section{Limitations}

While we provide a mathematical formulation for the variance propagation and the expected training behavior of deep $LDL^\top$ networks, we do not see the expected catastrophic decay. This paper attempts to explain the results' reasoning based on the logged information; however, no explicit conclusion was drawn, and further research is required. 

\section{Conclusion}

For the $LDL^\top$ normalization based on Cholesky normalization, the marginal variance was derived to derive the best initial parameterization. This ensures that variance propagation in a deep Lipschitz network is maintained, preventing the network from decaying to zero as layer sizes increase. However, from the current analysis, it is deemed impossible to ensure that the Lipschitz does not decay for deep feedforward $\Ll$-Lipschitz networks, as that implies that $\gamma = 1$, which, as derived above, makes it impossible for $\Var{y} = 1$, which would have to be the requirement in Lipschitz network initialization. However, empirical analysis found that, compared to traditional Kaiming or He initialization, a scaling of $1 / \sqrt{n}$ is less effective at reducing per-layer output variance. In contrast, a scaling of $10 / \sqrt{n}$ is more effective.

We noticed, however, that changing the initialization scale alone does not effectively improve training of these networks when using a modern optimizer, especially one from the Adam family, where gradient normalization helps mitigate issues with scale. In addition, the initialization scaling does not impact the network as expected in practice. We expected catastrophic decay of the network as the network depth increased; however, this was not the case, and in fact, validation AUC seemed invariant to depth. 

\section{Code}

The code for generating the symbolic moments and the code for computing the variance distribution can be found in  \href{https://github.com/Marius-Juston/LDLTLipschitzInitialization}{https://github.com/Marius-Juston/LDLTLipschitzInitialization}.

Additional figures illustrating the training dynamics with respect to additional metrics are available in the repository, but have been omitted due to space constraints. These can further help in understanding the network's internal dynamics as it trains.   

\subsubsection*{Declaration of Generative AI and AI-assisted technologies in the writing process}

During the preparation of this work, the authors used ChatGPT (OpenAI) to improve the clarity and fluency of the English text and to help find sources. After using this tool, the author reviewed and edited the content as needed and takes full responsibility for the publication's content.

% references section
\bibliographystyle{IEEEtran}
\bibliography{main}

\appendices

\section{Higher order Wishart moments} \label{sec:moments}

We continue the list of moments from Section \ref{sec:series_expansion} below from $k=6$ to $k=10$.

\begin{itemize}
   % \item $k =5$
   %  \begin{align*}
   %      \begin{bmatrix}1 & 0 & 0 & 0 & 0\\10 & 10 & 0 & 0 & 0\\65 & 55 & 20 & 0 & 0\\160 & 175 & 55 & 10 & 0\\148 & 160 & 65 & 10 & 1\end{bmatrix}
   %  \end{align*}
        \item $k =6$
    \begin{align*}
        \begin{bmatrix}1 & 0 & 0 & 0 & 0 & 0\\15 & 15 & 0 & 0 & 0 & 0\\155 & 135 & 50 & 0 & 0 & 0\\701 & 787 & 262 & 50 & 0 & 0\\1620 & 1827 & 787 & 135 & 15 & 0\\1348 & 1620 & 701 & 155 & 15 & 1\end{bmatrix}
    \end{align*}
        \item $k =7$
    \begin{align*}
        \begin{bmatrix}1 & 0 & 0 & 0 & 0 & 0 & 0\\21 & 21 & 0 & 0 & 0 & 0 & 0\\315 & 280 & 105 & 0 & 0 & 0 & 0\\2247 & 2569 & 889 & 175 & 0 & 0 & 0\\9324 & 10822 & 4844 & 889 & 105 & 0 & 0\\19068 & 23688 & 10822 & 2569 & 280 & 21 & 0\\15104 & 19068 & 9324 & 2247 & 315 & 21 & 1\end{bmatrix}
    \end{align*}
            \item $k =8$
\begingroup
\setlength\arraycolsep{2pt}
\begin{align*}
\resizebox{0.97\linewidth}{!}{$
            \begin{bmatrix}1 & 0 & 0 & 0 & 0 & 0 & 0 & 0\\28 & 28 & 0 & 0 & 0 & 0 & 0 & 0\\574 & 518 & 196 & 0 & 0 & 0 & 0 & 0\\5908 & 6846 & 2436 & 490 & 0 & 0 & 0 & 0\\38029 & 45070 & 20730 & 3985 & 490 & 0 & 0 & 0\\138016 & 175566 & 83280 & 20730 & 2436 & 196 & 0 & 0\\264420 & 343904 & 175566 & 45070 & 6846 & 518 & 28 & 0\\198144 & 264420 & 138016 & 38029 & 5908 & 574 & 28 & 1\end{bmatrix}
$}
\end{align*}
\endgroup  
        \item $k =9$, Given that all these matrices are persymmetric, to compress these matrices, we do not show the right symmetric second half
\begingroup
\setlength\arraycolsep{2pt}
\begin{align*}
\resizebox{0.97\linewidth}{!}{$
    \begin{bmatrix}
            1 & 0 & 0 & 0 & 0 & 0 & 0 & 0 & 0\\
            36 & 36 & 0 & 0 & 0 & 0 & 0 & 0 & \\
            966 & 882 & 336 & 0 & 0 & 0 & 0 &  & \\
            13524 & 15834 & 5754 & 1176 & 0 & 0 &  &  & \\
            124029 & 149346 & 70104 & 13941 & 1764 &  &  &  & \\
            692088 & 896238 & 437070 & 112575 &  &  &  &  &\\
            2325740 & 3092304 & 1628628 &  &  &  &  &  & \\
            4166880 & 5705232 &  &  &  &  &  &  & \\
            2998656 &  &  &  &  &  &  &  & 
            \end{bmatrix}
$}
\end{align*}
\endgroup
    % \end{tiny}
    \item $k =10$
    \begingroup
\setlength\arraycolsep{2pt}
\begin{align*}
\resizebox{0.97\linewidth}{!}{$
        \begin{bmatrix}
            1 & 0 & 0 & 0 & 0 & 0 & 0 & 0 & 0 & 0\\
            45 & 45 & 0 & 0 & 0 & 0 & 0 & 0 & 0 & \\
            1530 & 1410 & 540 & 0 & 0 & 0 & 0 & 0 &  & \\
            27930 & 32970 & 12180 & 2520 & 0 & 0 & 0 & &  & \\
            344961 & 420600 & 200580 & 40935 & 5292 & 0 &  &  &  & \\
            2723469 & 3576413 & 1781680 & 470920 & 60626 &  &  &  &  & \\
            13945700 & 18861430 & 10172240 & 2821775 &  &  &  &  &  & \\
            43448940 & 60666700 & 33822740 &  &  &  &  &  &  & \\
            74011488 & 105232400 &  &  &  &  &  &  &  & \\
            51290496 &  &  &  &  &  &  &  &  & 
        \end{bmatrix}
$}
\end{align*}
\endgroup
\end{itemize}

\section{Test AUC}

To better visualize the interaction between the initialization scale and the network's depth and the test AUC, we generate contour plots that approximate the discrete space as continuous ones using a Gaussian Process Regressor with the combination of a Matern kernel, to maintain sharpness in the distribution, and a White kernel, to account for noise in the distribution. 
Figure \ref{fig:test_auc_sgd_gpr} demonstrates the contour plot using the SGD optimizer. It illustrates that above a scale of 2, the network does not train and achieves a random guessing performance (AUC = 0.5). 
%
% \begin{figure}[ht!]
%     \centering
%     \includegraphics[width=1\linewidth]{Figures/Training/epoch/pdf/test_auc_sgd_gpr.pdf}
%     \caption{Contour plot of the test data set's AUC with respect to the initialization scale and the $LDL^\top$ network's depth. The network was trained using the SGD optimizer. The space is approximated as continuous using a Gaussian Process Regressor (GPR) with a Matern and White kernel combination. It can be noticed that the best AUC is achieved when the initialization scale is small.}
%     \label{fig:test_auc_sgd_gpr}
% \end{figure}
%
While the AdamW optimizer (Figure \ref{fig:test_auc_adamw_gpr}) performs much better and fails predominantly in regions of both deep networks (above 20 layers) and larger scale factors (around 6).
%
% \begin{figure}[ht!]
%     \centering
%     \includegraphics[width=1\linewidth]{Figures/Training/epoch/pdf/test_auc_adamw_gpr.pdf}
%     \caption{It can be observed that the network performs worse on networks with large initialization scales and deep network depths.}
%     \label{fig:test_auc_adamw_gpr}
% \end{figure}
%
\begin{figure}[ht!]
    \centering
    \begin{subfigure}[b]{1\linewidth}
    \includegraphics[width=1\textwidth]{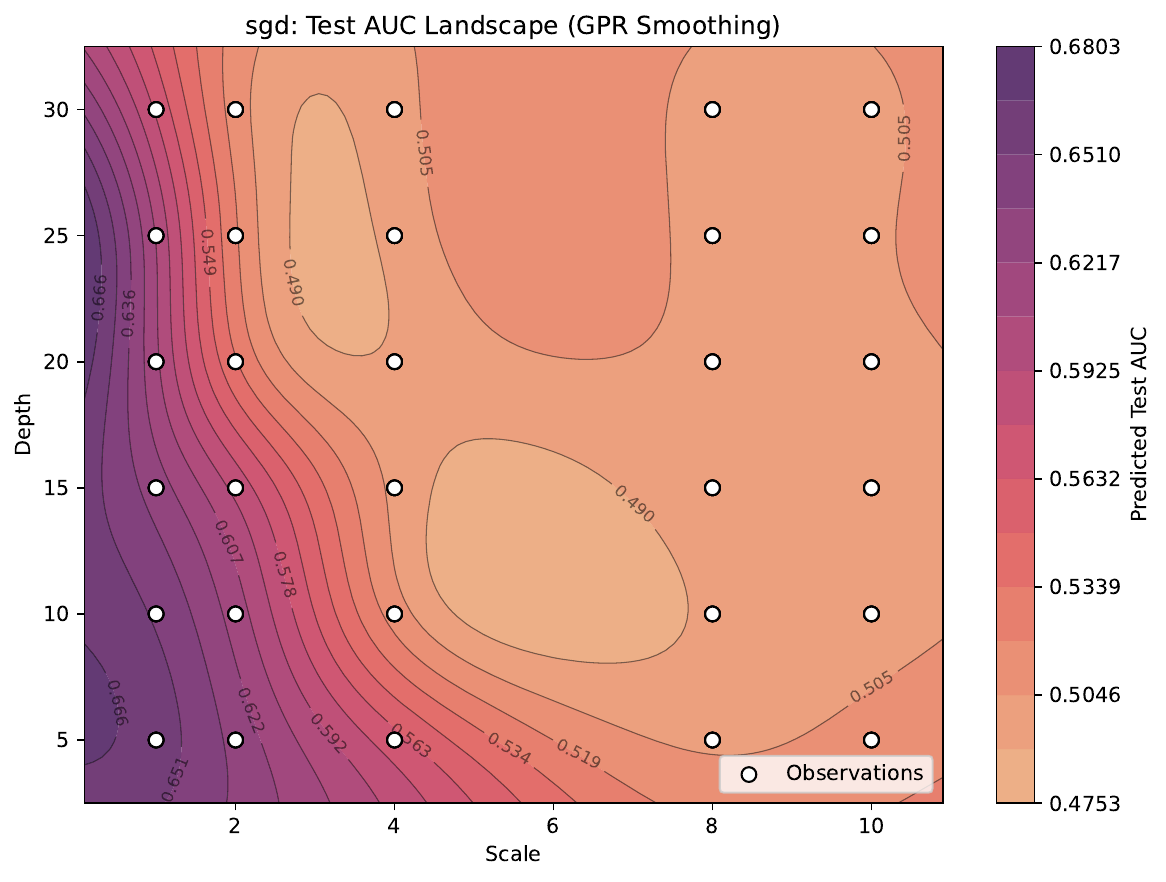}
    \caption{It can be noticed that the best AUC is achieved when the initialization scale is small.}
    \label{fig:test_auc_sgd_gpr}
\end{subfigure}
\hfill
    \begin{subfigure}[b]{1\linewidth}
    \includegraphics[width=1\textwidth]{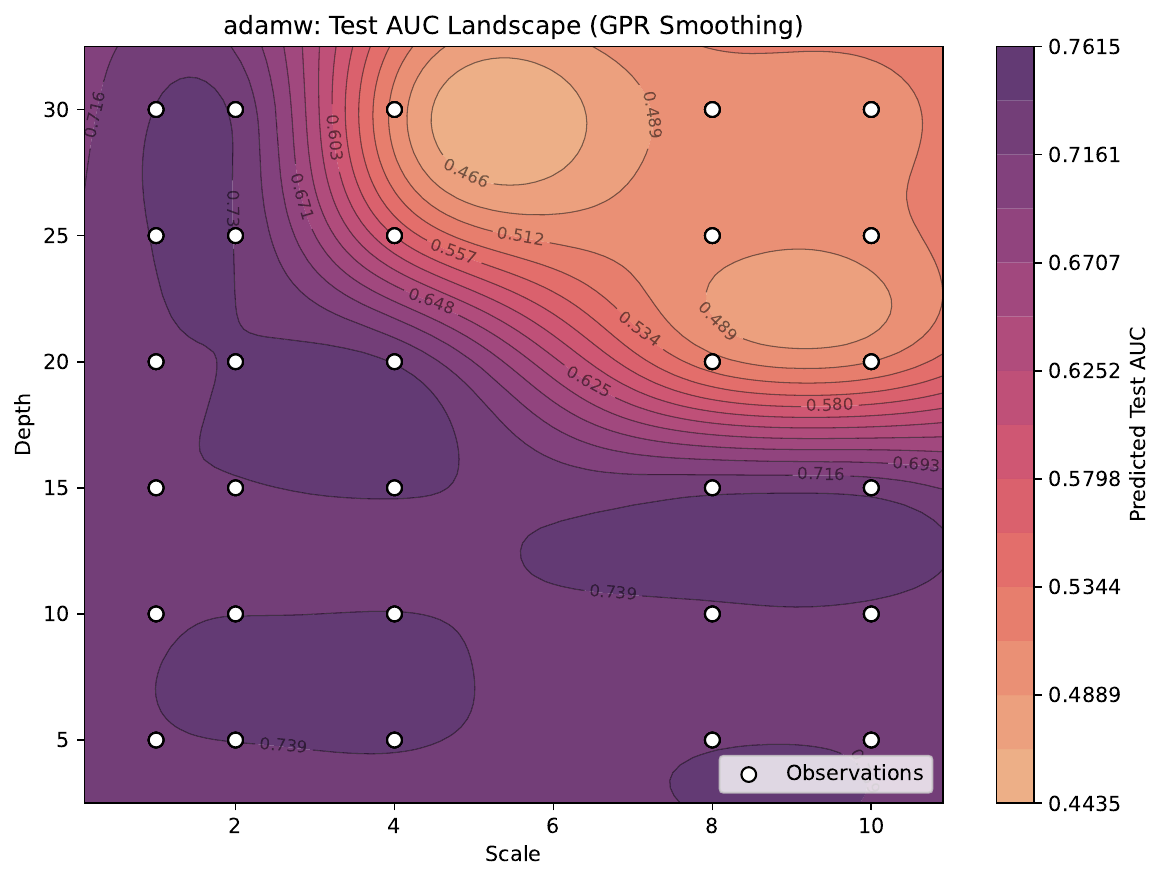}
    \caption{It can be observed that the network performs worse on networks with large initialization scales and deep network depths.}
    \label{fig:test_auc_adamw_gpr}
\end{subfigure}
    \caption{Contour plot of the test data set's AUC with respect to the initialization scale and the $LDL^\top$ network's depth. The network was trained using the SGD and AdamW optimizers, respectively. The space is approximated as continuous using a Gaussian Process Regressor (GPR) with a Matern and White kernel combination.}
    \label{fig:test_auc_gpr}
\end{figure}
\section{Input-Output CKA}

We further explore the usage of the Input Centered Kernel Alignment \cite{KornblithSimilarityRevisited}, computed as 
\begin{align}
    HSIC(K,L) = \frac{1}{(n-1)2}\trace{KHLH},
\end{align}
where $H$ is the centering matrix, $H_n = I_n - \frac{1}{n}\boldsymbol{1}\boldsymbol{1}^\top$. With the equivalent CKA metric, while being invariant to isotropic scaling,
\begin{align*}
    CKA(K,L)= \frac{HSIC(K,L)}{\sqrt{HSIC(K,K)HSIC(L, L)}}.
\end{align*}
The CKA metric is invariant to orthogonal transformation and isotropic scaling, but not invertible linear transformation. If CKA is close to 1, this implies a near-identity representation. Figures \ref{fig:input_output_cka_depth} and \ref{fig:input_output_cka_scale}, with respect to the network's depth and scale, show the computation of the CKA between the reference input features and the predicted output features. If a network outputs a CKA close to 1, it means the network has not learned and is just providing the input's identity. When exploring CKA of AdamW, we see that it achieves small values and converges within 40 epochs, without significant changes later on, and shows no relationship between the scale and depth hyperparameters.

In contrast, SGD shows that during training, it initially follows the same regime as AdamW; however, with a similar convergence time to AdamW, the input CKA increases from its initial position rather than decreasing as in AdamW. It is unclear why that is the case. In addition, we observe a general trend: the larger the initialization scale, the smaller the resulting CKA.
%
% \begin{figure}[ht!]
%     \centering
%     \includegraphics[width=1\linewidth]{Figures/Training/epoch/pdf/input_output_cka_depth.pdf}
%     \caption{Comparison of the input-output CKA training dynamics with respect to the $LDL\top$ network's depth between the SGD and the AdamW optimizers. The confidence intervals are computed using bootstrapping of 1000 samples. Illustrates AdamW converging at all depths around 40 epochs after starting, while SGD increases in value during its run.}
%     \label{fig:input_output_cka_depth}
% \end{figure}
% %
% \begin{figure}[ht!]
%     \centering
%     \includegraphics[width=1\linewidth]{Figures/Training/epoch/pdf/input_output_cka_scale.pdf}
%     \caption{Comparison of the input-output CKA training dynamics with respect to the initialization scale between the SGD and the AdamW optimizers. The confidence intervals are computed using bootstrapping of 1000 samples. Illustrates AdamW converging at all depths around 40 epochs after starting, while SGD increases in value during its run. Scale in SGD is related to CKA training.}
%     \label{fig:input_output_cka_scale}
% \end{figure}
%
\begin{figure}[ht!]
     \centering
     \begin{subfigure}[b]{1\linewidth}
         \centering
         \includegraphics[width=\textwidth]{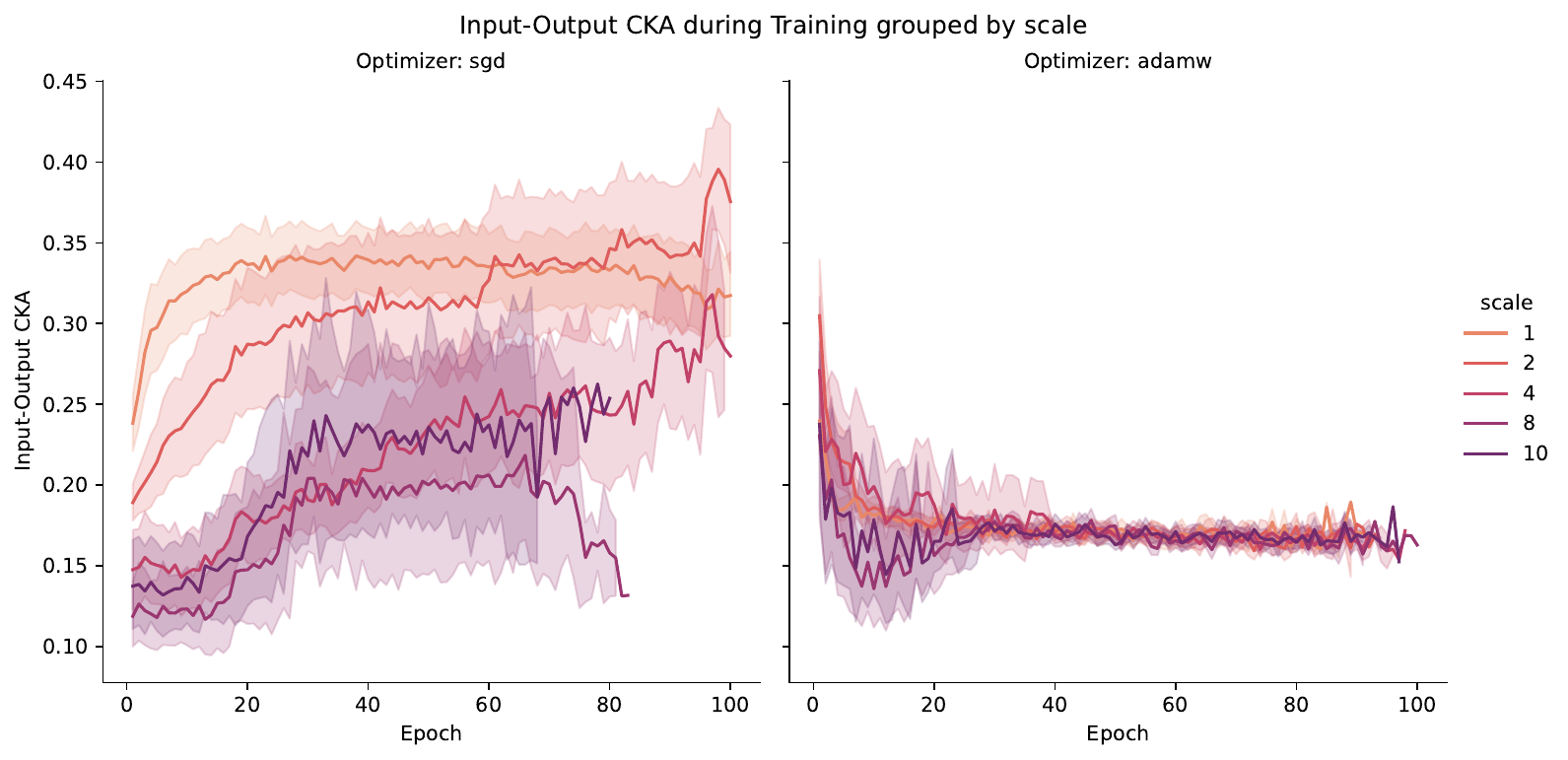}
         \caption{Illustrates AdamW converging at all depths around 40 epochs after starting, while SGD increases in value during its run. Scale in SGD is related to CKA training. }
         \label{fig:input_output_cka_scale}
     \end{subfigure}
     \hfill
     \begin{subfigure}[b]{1\linewidth}
         \centering
         \includegraphics[width=\textwidth]{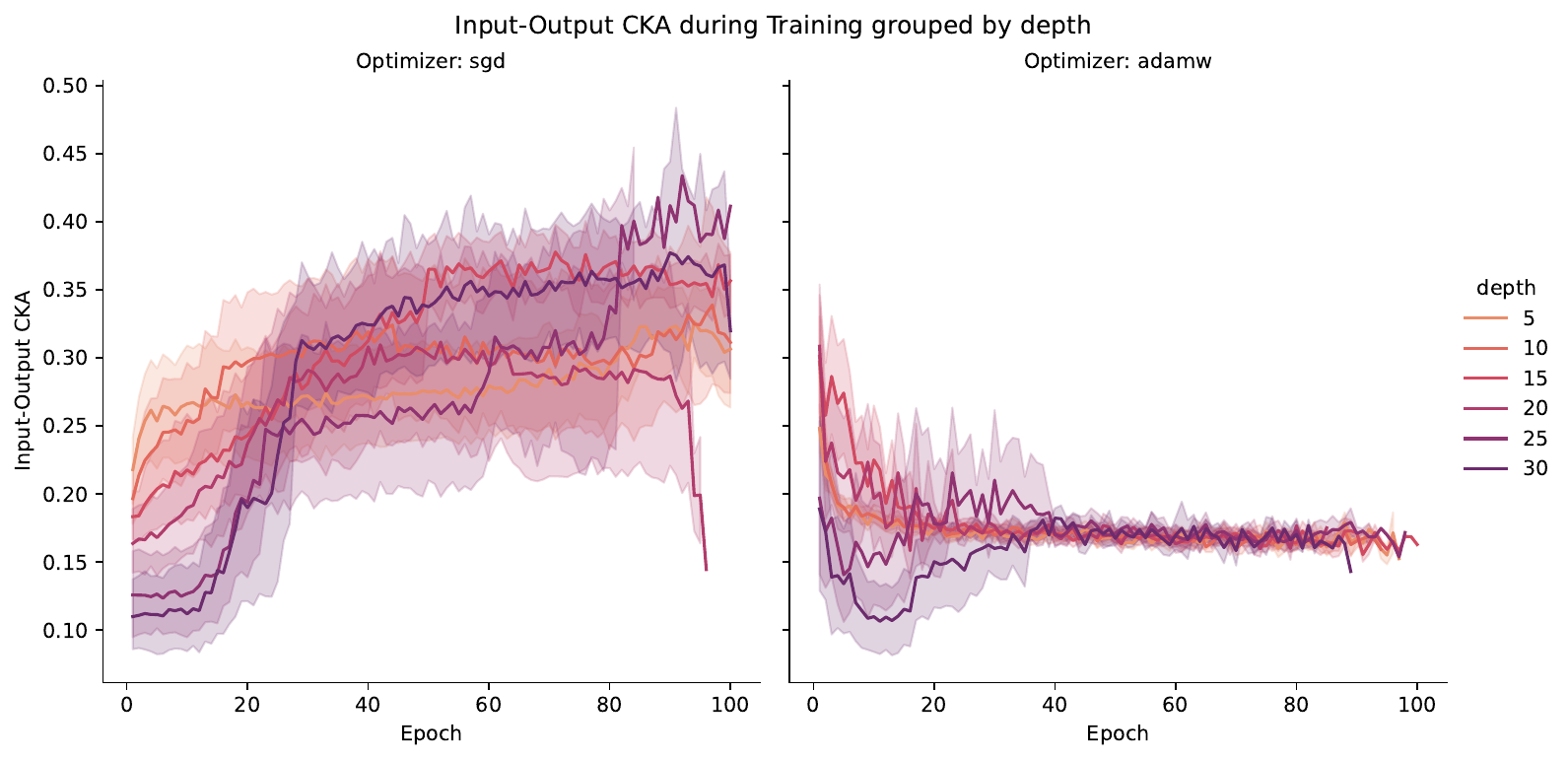}
         \caption{Illustrates AdamW converging at all depths around 40 epochs after starting, while SGD increases in value during its run.}
         \label{fig:input_output_cka_depth}
     \end{subfigure}
        \caption{Comparison of the input-output CKA training dynamics with respect to the initialization scale, Figures \ref{fig:input_output_cka_scale} and depth, respectively, Figure \ref{fig:input_output_cka_depth}, between the SGD and the AdamW optimizers. The confidence intervals are computed using bootstrapping of 1000 samples.}
        \label{fig:input_Cka}
\end{figure}
\section{Network Activation}

To verify that the network is not effectively dead due to activation propagation or other potential issues, we observe the stable rank, a continuous counterpart to matrix rank \cite{Ipsen2024StableMatrices},
\begin{align*}
    SR(W) = \frac{\|W\|_F^2}{\|W\|_2^2},
\end{align*}
From the training we notice that the stable rank of both the feature propoagatoin throughout the layers, Figure \ref{fig:rank_collapse_depth_evolution}, decays drastically to a single component by the end of it is, with only a slight improvement for smaller initialization scales, This is interesting as Figure \ref{fig:activation_dead_fraction_scale}, does not show a drastic number of dead neurons, due to the ELU activation function. Similarly, the gradient throughout training follows a very similar low-rank representation, as shown in Figure \ref{fig:error_signal_stable_rank_scale}. The reasoning warrants further study.
\begin{figure}[ht!]
    \centering
    \includegraphics[width=1\linewidth]{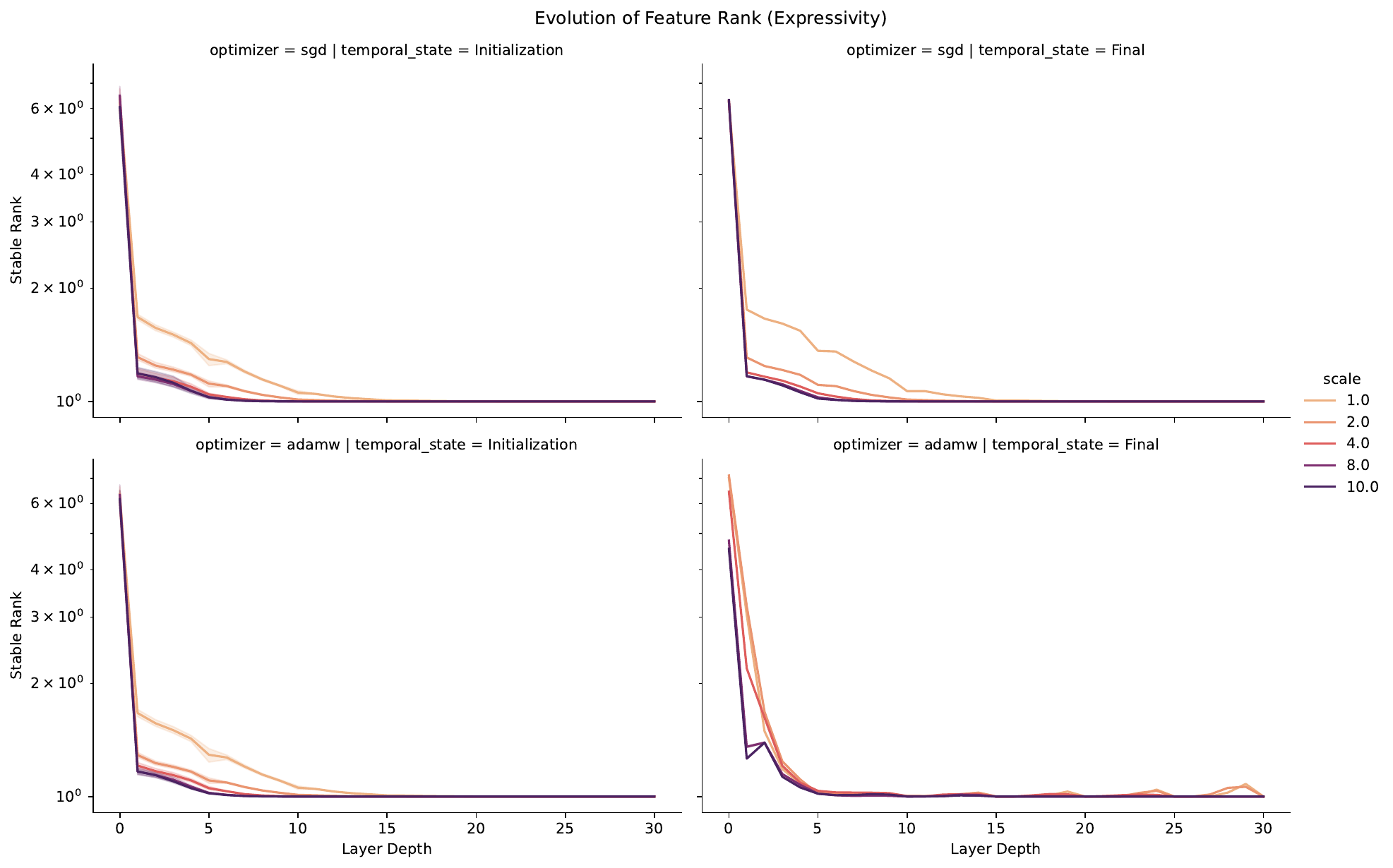}
    \caption{Stable rank between SGD and AdamW throughout the network layers at different scales, observed at the first epoch and the last epoch. This illustrates a drastic decay in rank after a couple of layers.}
    \label{fig:rank_collapse_depth_evolution}
\end{figure}
\begin{figure}[ht!]
    \centering
    \includegraphics[width=1\linewidth]{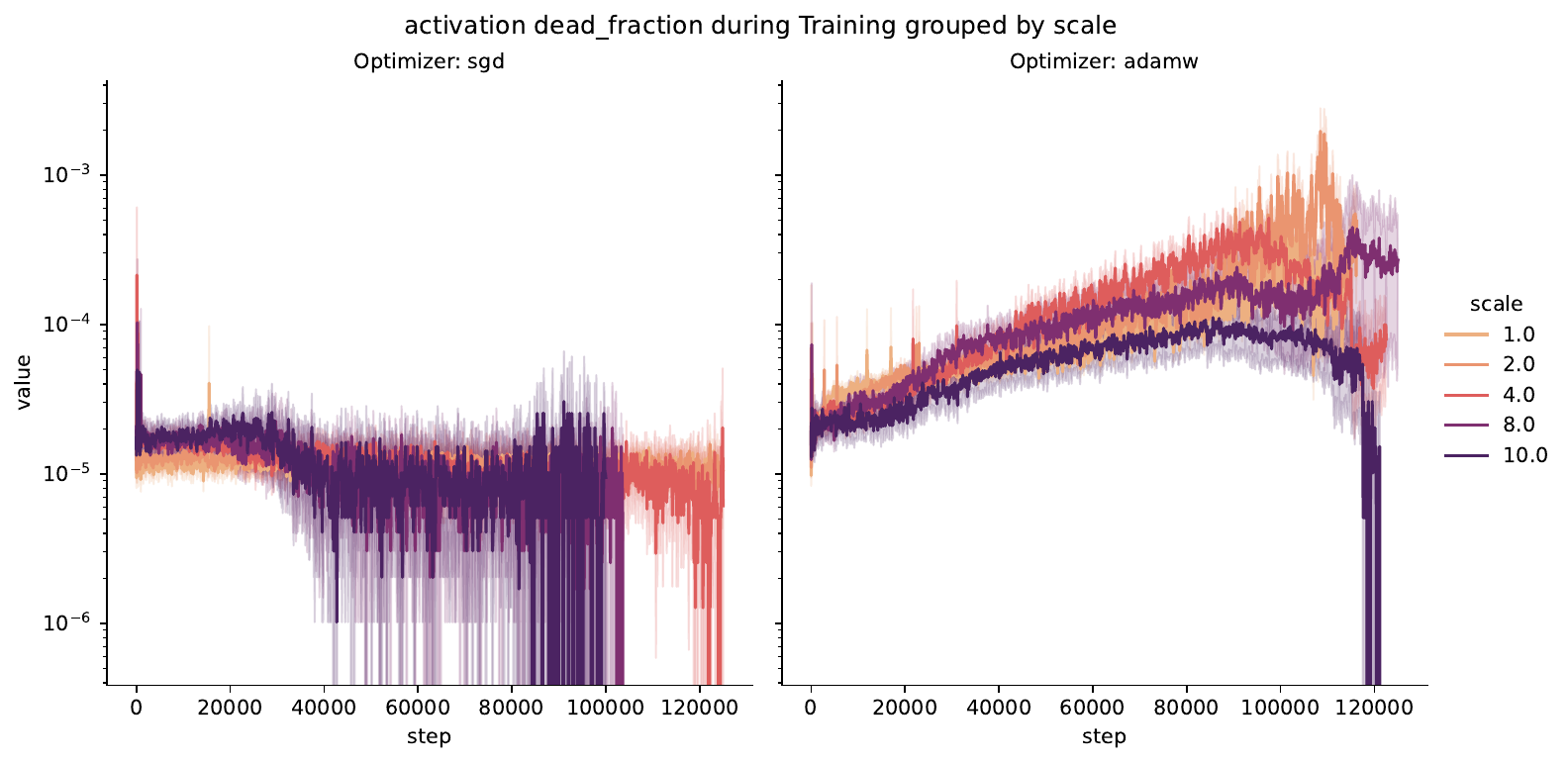}
    \caption{Percentage of dead neurons (a dead neuron has $|x| < 10^{-6}$), between optimizers. Only a minimal number are dead.}
    \label{fig:activation_dead_fraction_scale}
\end{figure}
\begin{figure}[ht!]
    \centering
    \includegraphics[width=1\linewidth]{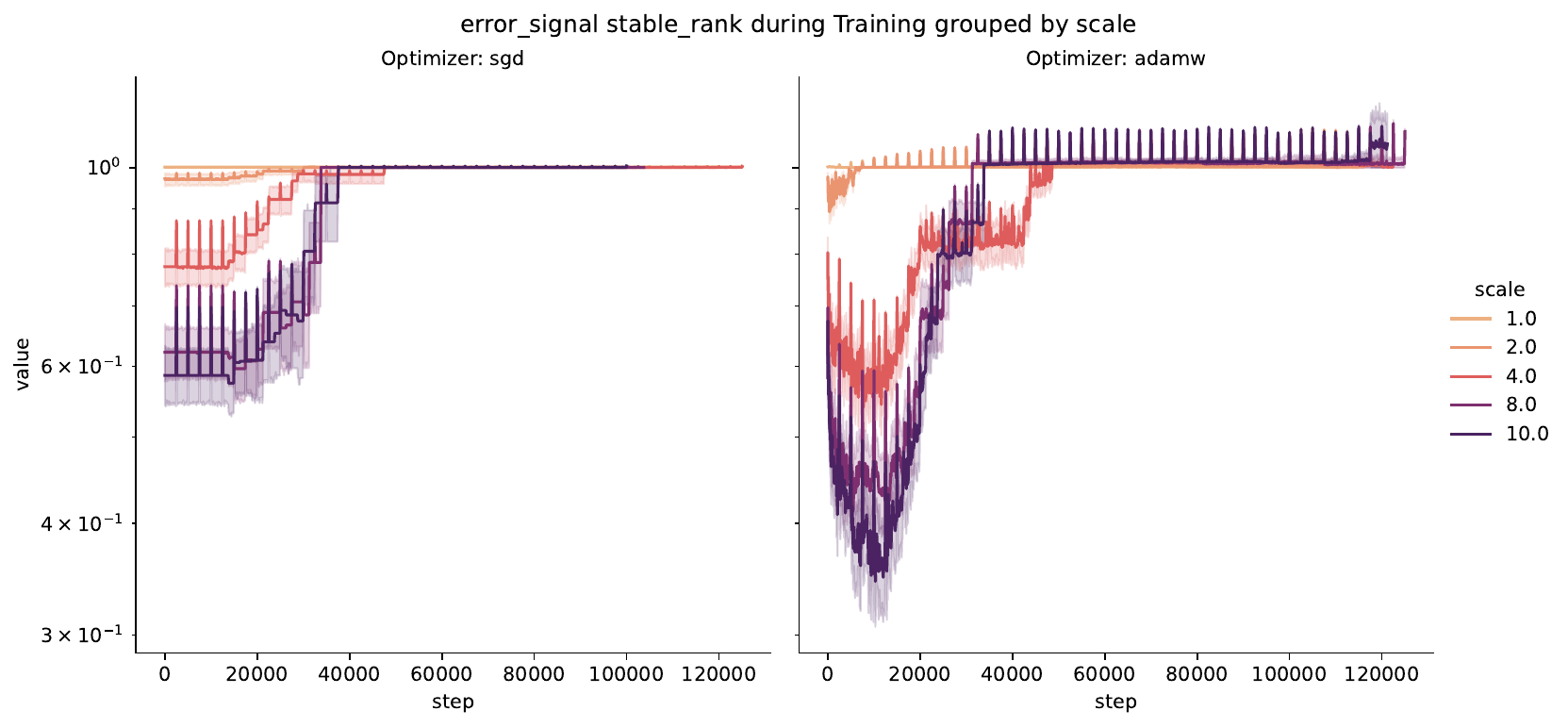}
    \caption{The stable rank of all the gradients with respect to scale. Illustrates rank collapse as ranks are one.}
    \label{fig:error_signal_stable_rank_scale}
\end{figure}
\end{document}